\documentclass[11pt]{article}
\usepackage{fullpage}
\usepackage{hyperref}       
\usepackage{url}    
\usepackage{booktabs}       
\usepackage{nicefrac}       
\usepackage{microtype}  
\usepackage{times}   
\usepackage{xcolor,colortbl} 
\usepackage{multirow}
\usepackage{tablefootnote}
\usepackage{subcaption}

\usepackage{preamble}
\usepackage{mathnotations}
\graphicspath{{Figures/}}

\usepackage[style=alphabetic,maxnames=12,maxbibnames=10,maxcitenames=10,maxalphanames=10,giveninits=true,doi=false,url=true]{biblatex}
\newcommand*{\citet}[1]{\AtNextCite{\AtEachCitekey{\defcounter{maxnames}{2}}} \textcite{#1}}

\newcommand*{\citep}[1]{\cite{#1}}

\newcommand{\alglinelabel}{%
  \addtocounter{ALC@line}{-1}
  \refstepcounter{ALC@line}
  \label
}

\begin{document}
\title{On Differentially  Private Federated Linear Contextual Bandits}
\author {
     Xingyu Zhou\thanks{Wayne State University, Detroit, USA.  Email: \texttt{xingyu.zhou@wayne.edu}}\quad 
    Sayak Ray Chowdhury\thanks{Microsoft Research, Bengaluru, Karnataka, India. Email: \texttt{t-sayakr@microsoft.com}  }
}

\date{}

\maketitle

\begin{abstract}
We consider cross-silo federated linear contextual bandit (LCB) problem under differential privacy, where multiple silos (agents) interact with the local users and communicate via a central server to realize collaboration while without sacrificing each user's privacy. We identify three issues in the state-of-the-art: (i) failure of claimed  privacy protection and (ii) incorrect regret bound due to noise miscalculation and (iii) ungrounded communication cost. 
To resolve these issues, we take a two-step principled approach. First, we design an algorithmic framework consisting of a generic federated LCB algorithm and flexible privacy protocols. Then, leveraging the proposed framework, we study federated LCBs under two different privacy constraints. We first establish privacy and regret guarantees under silo-level local differential privacy, which fix the issues present in state-of-the-art algorithm.
To further improve the regret performance, we next consider shuffle model of differential privacy, under which we show that our algorithm can achieve nearly ``optimal'' regret without a trusted server. 
We accomplish this via two different schemes --  one relies on a new result on privacy amplification via shuffling for DP mechanisms and another one leverages the integration of a shuffle protocol for vector sum into the tree-based mechanism, both of which might be of independent interest. Finally, we support our theoretical results with
numerical evaluations over contextual bandit instances generated from both synthetic and real-life data.
\end{abstract}

\newpage
\tableofcontents
\newpage

\section{Introduction}
We consider the classic \emph{cross-silo} Federated Learning (FL) paradigm~\cite{kairouz2021advances} applied to linear contextual bandits (LCB). In this setting, a set of $M$ local silos or agents (e.g., hospitals) communicate with a central server to learn about the unknown bandit parameter (e.g., hidden vector representing values of the
user for different medicines). In particular, at each round $t \in [T]$, each local agent $i \in [M]$ receives a new user (e.g., patient) with context information $c_{t,i} \in \cC_i$ (e.g., age, gender, medical history), recommends an action $a_{t,i} \in \cK_i$ (e.g., a choice of medicine), and then it observes a real-valued reward $y_{t,i}$ (e.g., effectiveness of the prescribed medicine). In linear contextual bandits, the reward $y_{t,i}$ is a linear function of the unknown bandit parameter $\theta^* \in \Real^d$ corrupted by $i.i.d$ mean-zero observation noise $\eta_{t,i}$, i.e., $y_{t,i} = \inner{x_{t,i}}{\theta^*} + \eta_{t,i}$, where $x_{t,i}= \phi_i(c_{t,i}, a_{t,i})$ and $\phi_i: \cC_i \times \cK_i \to \Real^d$ is a known function that maps a context-action pair to a $d$-dimensional real-valued feature vector. 
The goal of federated LCB is to minimize the cumulative \emph{group} pseudo-regret defined as 
\begin{align*}
    R_M(T) = \sum_{i=1}^M\sum_{ t=1}^T \left[\max_{a\in \cK_i} \inner{\phi_i(c_{t,i}, a)}{\theta^*} - \inner{x_{t,i}}{\theta^*}\right].
\end{align*}
To achieve the goal, as in standard cross-silo FL, the agents are allowed to communicate with the central server following a star-shaped communication, i.e., each agent can communicate with the server by uploading and downloading data, but agents cannot communicate with each other directly. However, the communication process (i.e., both data and schedule) could also possibly incur privacy leakage for each user $t$ at each silo $i$, e.g., the sensitive context information $c_{t,i}$ and reward $y_{t,i}$.

To address this privacy risk, we resort to \emph{differential privacy}~\cite{dwork2014algorithmic}, a principled way to prove
privacy guarantee against adversaries with arbitrary auxiliary information. 
In standard \emph{cross-device} FL, the notion of privacy is often the client-level DP, which protects the identity of each participating client or device. However, it has limitations in the setting of cross-silo FL, where the protection targets are users (e.g., patients) rather than participating silos or agents (e.g., hospitals). Also, in order to adopt client-level DP to cross-silo FL, one needs the server and other silos to be trustworthy, which is often not the case. Hence, recent studies~\cite{lowy2021private,lowy2022private, liu2022privacy,dobbe2018customized} on cross-silo federated supervised learning have converged to a new privacy notion, \emph{which 
requires that for each silo, all of its communication during the entire process is private (``indistinguishable'') with respect to change of one local user of its own.} This allows one to protect \emph{each user} within each silo without trustworthy server and other silos. In this paper, we adapt it to the setting of cross-silo federated contextual bandits and call it \emph{silo-level LDP}.


\citet{dubey2020differentially} adopt a similar but somewhat weaker notion of privacy called \emph{Federated DP} and takes the first step to tackle this important problem of private and federated linear contextual bandits (LCBs). In fact, the performance guarantees presented by the authors are currently the state-of-the-art for this problem. The proposed algorithm \emph{claims} to protect the privacy of each user at each silo. Furthermore, given a privacy budget $\epsilon > 0$, the claimed regret bound is $\widetilde{O}(\sqrt{MT/\epsilon})$ with only $O(M\log T)$ communication cost,
which matches the regret of a super-single agent that plays for total $MT$ rounds. Unfortunately, in spite of being the state-of-the-art, the aforementioned privacy, regret and communication cost guarantees have fundamental gaps, as discussed below.


\subsection{Our Contributions}
\textbf{Identify privacy, regret and communication gaps in state-of-the-art \cite{dubey2020differentially}.} 
In Section~\ref{sec:leakage}, we first show that the proposed algorithm in~\cite{dubey2020differentially} could leak privacy from the side channel of adaptive communication schedule, which depends on users' \emph{non-private} local data. Next, we identify a mistake in total injected privacy noise in the current regret analysis. Accounting for this miscalculation, the correct regret bound would amount to $\widetilde{O}(M^{3/4}\sqrt{T/\epsilon})$, which is 
$M^{1/4}$ factor higher than the claimed one, and doesn't match regret performance of the super agent. Finally, we observe that due to the presence of privacy noise, its current analysis for $O(M\log T)$ communication cost 
no longer holds. To resolve these issues, we take the following two-step principled approach:

\textbf{(i) design a generic algorithmic and analytical framework.} In Section~\ref{sec:alg}, we propose a generic federated LCB algorithm along with a flexible privacy protocol. Our algorithm adopts a fixed-batch schedule (rather than an adaptive one in~\cite{dubey2020differentially}) that helps avoid privacy leakage from the side channel, as well as subtleties in communication analysis. Our privacy protocol builds on a distributed version of the celebrated tree-based algorithm~\cite{chan2011private,dwork2010differential}, enabling us to provide different privacy guarantees in a unified way. 
We further show that our algorithm enjoys a simple and generic analytical regret bound that only depends on the total amount of injected privacy noise under the required privacy constraints. 

\textbf{(ii) prove performance guarantees under different privacy notions.} We build upon the above framework to study federated LCBs under two different privacy constraints. In Section~\ref{sec:main-LDP}, we consider silo-level LDP (a stronger notion of privacy than Federated DP of~\cite{dubey2020differentially}) and establish privacy guarantee with a correct regret bound $\widetilde{O}(M^{3/4}\sqrt{T/\epsilon})$ and communication cost $O(\sqrt{MT})$, hence fixing  the gaps in~\cite{dubey2020differentially}. Next, to match the regret of a super single agent, we consider shuffle DP (SDP)~\cite{cheu2019distributed} in Section~\ref{sec:main-SDP} and establish a regret bound of $\widetilde{O}(\sqrt{MT/\epsilon})$. We provide two different techniques to achieve this -- one that relies on a new result on privacy amplification via shuffling for DP mechanisms and the other that integrates a shuffle protocol for vector sums \cite{cheu2021shuffle} into the tree-based mechanism, both of which might be of independent interest.
In Section~\ref{sec:exp}, we support our theoretical results with
simulations on contextual bandit instances generated from synthetic and real-life data.



\section{Related Work}
Private bandit learning has recently received increasing attention under various notion of DP. 
For multi-armed bandits (MAB) where rewards are the sensitive data, different DP models including the central model~\cite{mishra2015nearly,azize2022privacy,sajed2019optimal}, local model~\cite{ren2020multi} and distributed model~\cite{chowdhury2022distributed,tenenbaum2021differentially} have been studied. Among them, we note that~\cite{chowdhury2022distributed} also presents optimal private regret bounds under the above three DP models while only relying on discrete privacy noise, hence avoiding the privacy leakage of continuous privacy noise on finite computers due to floating point arithmetic. 
For linear bandits (without contexts protection),~\cite{li2022differentially} establishes the first near-optimal private regret bounds for central, local, and shuffle models of approximate DP. The same problem has also been studied under pure-DP in~\cite{hanna2022differentially}. 
In the specific case of linear contextual bandits, where both the contexts and rewards need to be protected, there are recent line of work under the central~\cite{shariff2018differentially}, local~\cite{zheng2020locally} and shuffle model~\cite{pmlr-v162-chowdhury22a,garcelon2022privacy,tenenbaum2023concurrent} of DP. 
Private bandit learning has also been studied beyond linear settings, such as kernel bandits~\cite{Zhou_Tan_2021,dubey2021no,li2023private}.

All the above papers consider learning by a single agent.
To the best of our knowledge,~\citet{dubey2020differentially} is the first to consider cross-silo federated linear contextual bandits (LCBs). Non-private federated or distributed LCBs have also been well studied~\cite{wang2019distributed,he2022simple,huang2021federated}. One common goal is to match the regret achieved by a super single agent that plays $MT$ rounds while keeping communication among agents as low as possible. Our work shares the same spirit in that we aim to match the  regret achieved by a super single agent under differential privacy.

Broadly speaking, our work also draws inspiration from recent advances in private cross-silo federated supervised learning~\cite{lowy2021private,liu2022privacy}. In particular, our silo-level local and shuffle DP definitions for federated LCBs in the main paper can be viewed as counterparts of the ones proposed for cross-silo federated supervised learning (see, e.g., ~\citet{lowy2021private}). 

\section{Differential Privacy in Federated LCBs}
\label{sec:DP}

We now formally introduce differential privacy in cross-silo federated contextual bandits. Let 
a dataset $D_i$ at each silo $i$ be given by a sequence of $T$ \emph{unique} users $U_{1,i}, \ldots, U_{T,i}$.
Each user $U_{t,i}$ is identified by her context information $c_{t,i}$ as well as reward responses she would give to all possible actions recommended to her. We say two datasets $D_i$ and $D_i'$ at silo $i$ are adjacent if they differ exactly in one participating user, i.e., $U_{\tau,i} \neq U'_{\tau,i}$ for some $\tau \in [T]$ and $U_{s,i} = U'_{s,i}$ for all $s \neq \tau$.

\textbf{Silo-level local differential privacy (LDP).}
Consider a multi-round, cross-silo 
federated learning algorithm $\cQ$. At each round $t$, each silo $i$ communicates a randomized message $Z_i^t$
of its data $D_i$ to the server, which may depend (due to collaboration) on previous randomized messages $Z_{j}^1,\ldots, Z_{j}^{t-1}$ from all other silos $j\neq i$. We allow $Z_i^{t}$ to be empty if there is no communication at round $t$. 
Let $Z_i = (Z_i^1,\ldots,Z_i^T)$ denote the full transcript of silo $i$'s communications with the server over $T$ rounds and $\cQ_i$ the induced local mechanism in this process. Note that $Z_i$ is a realization of random messages generated according to the local mechanism $\cQ_i$. We denote by $Z_{- i}=(Z_1,\ldots,Z_{i-1},Z_{i+1},\ldots,Z_M)$ the full transcripts of all but silo $i$. We assume that $Z_i$ is conditionally independent of $D_j$ for all $j \neq i$ given $D_i$ and $Z_{-i}$. With this notation, we have the following definition of silo-level LDP.

\begin{definition}[Silo-level LDP]
\label{def:silo-LDP-general}
A cross-silo federated learning algorithm $\cQ$ with $M$ silos
is said to be $(\epsilon_i,\delta_i)_{i \in M}$ silo-level LDP if for each silo $i \!\in\! [M]$, it holds that
\begin{align*}
 \mathbb{P}\Big[\cQ_i(Z_i \!\in\! \cE_i | D_i, \!Z_{-i})\Big] \!\le\! e^{\epsilon_i} \mathbb{P}\Big[\cQ_i(Z_i \!\in\! \cE_i | D'_i, \!Z_{-i})\Big] \!+\! \delta_i~,
\end{align*}
for all adjacent datasets $D_i$ and $D'_i$, and
for all events $\cE_i$ in the range of $\cQ_i$. 
If $\epsilon_i = \epsilon$ and $\delta_i=\delta$ for all $i \in [M]$, we simply say $\cQ$ is $(\epsilon,\delta)$-silo-level LDP.
\end{definition}



Roughly speaking, a silo-level LDP algorithm protects the privacy of each individual user (e.g., patient) within each silo in the sense that an adversary (which could either be the central server or other silos) cannot infer too much about any individual's sensitive information (e.g., context and reward) or determine whether an individual participated in the learning process.\footnote{This is indeed a notion of item-level DP. It appears under different names in prior work, e.g., silo-specific sample-level DP~\cite{liu2022privacy}, inter-silo record-level DP~\cite{lowy2021private}. A comparison of this notion of privacy with standard local DP, central DP and shuffle DP for single-agent LCBs is presented in Appendix~\ref{app:ldp}.
}


\begin{remark}[Federated DP vs. Silo-level LDP]
\label{rem:fed-dp}
\citet{dubey2020differentially} consider a privacy notion called Federated DP (Fed-DP in short). 
As summarized in~\cite{dubey2020differentially}, Fed-DP requires ``the action chosen by any agent must be sufficiently impervious (in probability) to any single pair $(x,y)$ from any other agent''. Both silo-level LDP and Fed-DP are item-level DP as the neighboring relationship is defined by differing in one participating user. The key here is to note that silo-level DP implies Fed-DP by the post-processing property of DP, and thus it is a stronger notion of privacy.
In fact, \citet{dubey2020differentially} claim to achieve Fed-DP by relying on privatizing the communicated data from each silo. However, as we shall see in Section~\ref{sec:leakage}, its proposed algorithm fails to privatize the adaptive synchronization schedule, which is the key reason behind privacy leakage in their algorithm.

\end{remark}


\textbf{Shuffle differential privacy (SDP).}
Next, we consider the notion of SDP~\cite{cheu2019distributed}, which builds upon a trusted third-party (shuffler) to amplify privacy. This provides us with the possibility to achieve a better regret compared to the one under silo-level LDP while still without a trusted server. Under the shuffle model of DP in FL, each silo $i\in [M]$ first applies a local randomizer $\cR$ to its raw local data and sends the randomized output to a shuffler $\cS$. The shuffler $\cS$ permutes all the messages from all $M$ silos uniformly at random and sends those to the central server. Roughly speaking, SDP requires all the messages sent by the shuffler to be private (``indistinguishable'') with respect to a single user change among all $MT$ users. This item-level DP is defined formally as follows.

 \begin{definition}[SDP]
Consider a cross-silo federated learning algorithm $\cQ$ that induces a (randomized) mechanism $\cM$ whose output is the collection of all messages sent by the shuffler during the entire learning process. Then, the algorithm $\cQ$ is said to be $(\epsilon,\delta)$-SDP if 
\begin{align*}
    \mathbb{P}\Big[\cM(D) \in \cE \Big] \le e^{\epsilon}\, \mathbb{P}\Big[\cM(D') \in \cE \Big] + \delta~,
\end{align*}
for all $\cE$ in the range of $\cM$ and for all adjacent datasets $D = (D_1,\ldots, D_M)$ and $D' = (D'_1,\ldots, D'_M)$ such that $\sum_{i=1}^M \sum_{t=1}^T \mathbbm{1}_{\{ U_{t,i} \neq U'_{t,i}\}} = 1$.
\end{definition}



\section{Privacy, Regret and Communication Gaps in State-of-the-Art}\label{sec:leakage}

In this section, we discuss the gaps present in privacy, regret and communication cost guarantees of the state-of-the-art algorithm proposed in \citet{dubey2020differentially}.

\subsection{Gap in Privacy Analysis}
We take a two-step approach to demonstrate the privacy issue in~\cite{dubey2020differentially}.
To start with, we argue that the proposed technique (i.e., Algorithm 1 in~\cite{dubey2020differentially}) fails to achieve silo-level LDP due to privacy leakage through the side channel of communication schedule (i.e., when the agents communicate with the server).
The key issue is that the adaptive communication schedule in their proposed algorithm depends on users' \emph{non-private} data. This fact can be utilized by an adversary or malicious silo $j$ to infer another silo $i$'s users' sensitive information, which violates the requirement of silo-level LDP.  
Specifically, in the proposed algorithm of~\cite{dubey2020differentially}, \emph{all} silos communicate with the server (which is termed as \emph{synchronous} setting) if
\begin{align}
\label{eq:sync}
    \exists\; \text{some silo}\; i \in [M]:f(X_i, Z) > 0~,
\end{align}
where $f$ is some function, $X_i$ is the non-private local data of silo $i$ since the last synchronization and $Z$ is all previously synchronized data. Crucially, the form of $f$ and the rule \eqref{eq:sync} are public information, known to all silos even before the algorithm starts.
This local and non-private data-dependent communication rule in~\eqref{eq:sync} causes privacy leakage, as illustrated below with a toy example.

\begin{example}[Privacy leakage]
\label{ex:leak}
Consider there are two silos $i$ and $j$ following the algorithm in~\cite{dubey2020differentially}. After the first round, $X_i$ in~\eqref{eq:sync} includes the data of the first user in silo $i$ (say Alice), $X_j$ includes the data of the first user in silo $j$ (say Bob) and $Z$ is empty (zero). Let communication be triggered at the end of first round and assume $f(X_j,0) \le 0$. Since the rule \eqref{eq:sync} is public, silo $j$ can infer that $f(X_i,0) > 0$, i.e. the communication is triggered by silo $i$. Since $f$ is also public knowledge, silo $j$ can utilize this to infer some property of $X_i$. Hence, by observing the communication signal \emph{only} (even without looking at the data), silo $j$ can infer some sensitive data of Alice.\footnote{ In fact, given the specific form of $f$ in~\cite{dubey2020differentially}, silo $j$ gets to know that $\log\det\left(I +\lambda_{\min}^{-1} x_{1,i}x_{1,i}^{\top} \right) > D$, where $\lambda_{\min} > 0$ is a regularizer (which depends on privacy budgets $\epsilon,\delta$) and $D > 0$ is some suitable threshold (see Appendix~\ref{app:gap} for the specific form of $f$). This in turn implies that $\norm{x_{1,i}} > C$, where $C$ is some constant. Since $x_{1,i}$ contains the context information of the user, this 
information could immediately reveal that some specific features in the context vector are active, which can be inferred by the adversary silo (e.g., silo $j$).}
\end{example}

The above example demonstrates that the proposed algorithm in~\cite{dubey2020differentially} does not satisfy silo-level LDP, implying (i) their current proof for Fed-DP guarantee via post-processing of silo-level LDP does not hold anymore and (ii) Fed-DP is weak privacy protection. However, it does not necessarily imply that this algorithm also does not satisfy the weaker notion of Fed-DP (as considered in~\cite{dubey2020differentially}). Nevertheless, one can show that this algorithm indeed also fails to guarantee Fed-DP by leveraging Example~\ref{ex:leak}.

To see this, recall the definition of Fed-DP from Remark~\ref{rem:fed-dp}. In the context of Example~\ref{ex:leak}, it translates to silo $j$ selecting similar actions for its users when a single user in silo $i$ changes. Specifically, if the first user in silo $i$ changes from Alice to say, Tracy, Fed-DP mandates that all $T$ actions suggested by silo $j$ to its local $T$ users remain “indistinguishable”. This, in turn, implies that the communicated data from silo $i$ must remain “indistinguishable” at silo $j$ for each $t \!\in\! [T]$. This is because the actions at silo $j$ are chosen \emph{deterministically} based on its local data as well as communicated data from silo $i$,  and the local data at silo $j$ remains unchanged.
However, in Algorithm 1 of~\cite{dubey2020differentially}, the communicated data from silo $i$ is not guaranteed to remain “indistinguishable” as synchronization depends on
non-private local data (e.g. $X_i$ in ~\eqref{eq:sync}). In other words, without additional privacy noise added to $X_i$ in~\eqref{eq:sync}, the change from Alice to Tracy could affect the \emph{existence of synchronization} at round $t\ge1$ a lot. Consequently, under these two neighboring situations (e.g. Alice vs. Tracy), the communicated data from silo $i$ could differ significantly at round $t+1$. 
As a result, the action chosen at round $t+1$ in silo $j$ can be totally different, which violates the Fed-DP definition. This holds true even if silo $i$ injects noise while communicating its data (as done in Algorithm 1 of~\cite{dubey2020differentially}) due to a large change of non-private communicated data (see Appendix~\ref{app:gap}).

\subsection{Gaps in Regret and Communication Analysis}
We now turn to regret and communication analysis of~\cite{dubey2020differentially}, which has fundamental gaps that lead to incorrect conclusions in the end. 
{First}, the reported cost of privacy in regret bound
is $\Tilde{O}(\sqrt{MT/\epsilon})$ (ignoring dependence on dimension $d$ for simplicity), which leads to the (incorrect) conclusion that federated LCBs across $M$ silos under silo-level LDP can achieve the same order of regret as a super single agent that plays $MT$ rounds. 
However, in the proposed analysis, the total amount of injected privacy noise is miscalculated. In particular, variance of total noise needs to be $M\sigma^2$ rather than the proposed value of $\sigma^2$. This comes from the fact that each silo injects Gaussian noise with variance $\sigma^2$ when sending out local data and hence the total amount of noise at the server is $M\sigma^2$. Accounting for this correction, the cost of privacy becomes $\Tilde{O}(M^{3/4}\sqrt{T/\epsilon})$, which is $O(M^{1/4})$ factor worse than the claimed cost. Hence, we conclude that Algorithm 1 in~\cite{dubey2020differentially} cannot achieve the same order of regret as a super single agent.
{Second}, the proposed analysis in \cite{dubey2020differentially} to show $O(\log T)$ communication cost for the \emph{data-adaptive} schedule \eqref{eq:sync} under privacy constraint essentially follows from the non-private analysis of~\cite{wang2019distributed}. Unfortunately, due to additional privacy noise, this direct approach no longer holds, and hence the reported logarithmic communication cost stands ungrounded (see Appendix~\ref{app:gap} for more details on this). 




\section{Our Approach}
\label{sec:alg}

To address all three issues in~\cite{dubey2020differentially}, we introduce a generic algorithm for private and federated linear contextual bandits (Algorithm~\ref{alg:FedLUCB-SDP}) along with a flexible privacy protocol (Algorithm~\ref{alg: p-sdp}), which not only allows us to present the correct privacy, regret, and communication results under silo-level LDP (and hence under Fed-DP) (Section~\ref{sec:main-LDP}), but also helps us achieve the same order of regret as a super single agent under SDP  (Section~\ref{sec:main-SDP}).
Throughout the paper, we make following assumptions.

\begin{assumption}[Boundedness~\cite{shariff2018differentially,pmlr-v162-chowdhury22a}]
\label{ass:bounded}
The rewards are bounded, i.e., $y_{t,i} \in [0,1]$ for all $t\in [T]$ and $i\in [M]$. Moreover, the parameter vector and the context-action features have bounded norms, i.e., $\norm{\theta^*}_2 \le 1$ and $\sup_{c,a}\norm{\phi_i(c,a)}_2 \le 1$ for all $i \in [M]$.
\end{assumption}

\subsection{Algorithm: Private Federated LinUCB}
\label{sec:algo}

\begin{algorithm}[!t]
  \caption{Private-FedLinUCB }
  \label{alg:FedLUCB-SDP}
\begin{algorithmic}[1]
  \STATE {\bfseries Parameters:} Batch size $B \in \mathbb{N}$, regularization $\lambda >0$, confidence radii $\{\beta_{t,i}\}_{t\in [T], i\in [M]}$, feature map $\phi_i:\cC_i \times \cK_i \to \Real^d$, privacy protocol $\cP = (\cR,\cS,\cA)$
  \STATE {\bfseries Initialize:} $W_{i} = 0, U_{i} = 0$ for all agents $ i \in [M]$, $\widetilde{W}_{\text{syn}} = 0$, $\widetilde{U}_{\text{syn}} = 0$ 
 \FOR{$t\!=\!1, \ldots, T$}
    \FOR{each agent $i = 1,\ldots, M$}
      \STATE Receive context $c_{t,i}$; compute $V_{t,i} =  \lambda I+ \widetilde{W}_{\text{syn}} + W_{i}$ and $\hat{\theta}_{t,i} = V_{t,i}^{-1}(\widetilde{U}_{\text{syn}} + U_{i})$ \alglinelabel{line:data}
      \STATE Play action
      $a_{t,i} \!=\! \argmax_{a \in \cK_i} \inner{\phi_i(c_{t,i},\! a)} {\hat{\theta}_{t,i}}\! + \!\beta_{t,i} \!\norm{\phi_i(c_{t,i},\!a)}_{V_{t,i}^{-1}}$; observe reward $y_{t,i}$\alglinelabel{line:beta}
      \STATE Set $x_{t,i} \!=\! \phi_i(c_{t,i}, a_{t,i})$, $U_{i} = U_{i} + x_{t,i}y_{t,i}$ and $W_{i} = W_{i} + x_{t,i}x_{t,i}^{\top}$\alglinelabel{line:UW}
      \ENDFOR
      \IF{$t \Mod B = 0$} \alglinelabel{line:com}
       \STATE{\textcolor{DarkBlue}{\texttt{{// Local randomizer $\cR$ at \emph{all} agents $i \in [M]$ }}}}
        \STATE Send randomized messages $R_{t,i}^{\text{bias}} = \cR^{\text{bias}}(U_i)$ and $R_{t,i}^{\text{cov}} = \cR^{\text{cov}}(W_i)$ to $\cS$
         \STATE{\textcolor{DarkBlue}{\texttt{// Third party $\cS$ }}}
         \STATE Shuffle (or, not) all messages $S_t^{\text{bias}} = \cS(\{R_{t,i}^{\text{bias}}\}_{i \in [M]})$ and $S_t^{\text{cov}} = \cS(\{R_{t,i}^{\text{cov}}\}_{i \in [M]})$
          \STATE{\textcolor{DarkBlue}{\texttt{// Analyzer $\cA$ at the server}}}
         \STATE Compute private synchronized statistics $\widetilde{U}_{\text{syn}}= \cA^{\text{bias}}(S_t^{\text{bias}})$ and $\widetilde{W}_{\text{syn}}= \cA^{\text{cov}}(S_t^{\text{cov}})$
         \STATE{\textcolor{DarkBlue}{\texttt{// \emph{All} agents $i \in [M]$}}}
        \STATE Receive $\widetilde{W}_{\text{syn}}$ and $\widetilde{U}_{\text{syn}}$ from the server and reset $W_{i} = 0$, $U_{i} = 0$\alglinelabel{line:sync}
      \ENDIF  
  \ENDFOR
\end{algorithmic}
\end{algorithm}

We build upon the celebrated LinUCB algorithm~\cite{abbasi2011improved} by adopting a \emph{fixed-batch} schedule for synchronization among agents and designing a privacy protocol $\cP$ (Algorithm~\ref{alg: p-sdp}) for both silo-level LDP and SDP .
At each round $t$, each agent $i$ recommends an action $a_{t,i}$ to each local user following \emph{optimism in the face of uncertainty} principle. First, the agent computes a local estimate $\hat{\theta}_{t,i}$ based on \emph{all} available data to her, which includes previously synchronized data from all agents as well as her own new local data (\textcolor{blue}{line}~\ref{line:data} of Algorithm~\ref{alg:FedLUCB-SDP}). Then, the action $a_{t,i}$ is selected based on the LinUCB decision rule (\textcolor{blue}{line}~\ref{line:beta}), where a proper radius $\beta_{t,i}$ is chosen to balance between exploration and exploitation. After observing the reward $y_{t,i}$, each agent accumulates her own local data (bias vector $x_{t,i}y_{t,i}$ and covariance matrix $x_{t,i}x_{t,i}^{\top}$) and stores them in $U_i$ and $W_i$, respectively (\textcolor{blue}{line}~\ref{line:UW}). 
A communication is triggered between agents and central server whenever a batch ends -- we assume w.l.o.g. total rounds $T$ is divisible by batch size $B$ (\textcolor{blue}{line}~\ref{line:com}). During this process, a protocol $\cP = (\cR,\cS,\cA)$ assists in aggregating local data among all agents while guaranteeing privacy properties (to be discussed in detail soon). After communication, each agent receives latest synchronized data $\widetilde{W}_{\text{syn}}, \widetilde{U}_{\text{syn}}$ from the server (\textcolor{blue}{line}~\ref{line:sync}). Here, for any $t \!=\! kB, k\in[T/B]$, $\widetilde{W}_{\text{syn}}$ represents noisy version of all covariance matrices up to round $t$ from all agents (i.e., $\sum_{i=1}^M\sum_{s=1}^t x_{s,i}x_{s,i}^{\top}$) and similarly, $\widetilde{U}_{\text{syn}}$ represents noisy version of all bias vectors $\sum_{i=1}^M\sum_{s=1}^t x_{s,i}y_{s,i}$. Finally, each agent resets $W_i$ and $U_i$ so that they can be used to accumulate new local data for the next batch. Note that Algorithm~\ref{alg:FedLUCB-SDP} uses a fixed-batch (data-independent) communication schedule rather than the adaptive, data-dependent one in~\cite{dubey2020differentially}. This allows us to resolve privacy and communication issues in~\cite{dubey2020differentially} (to be discussed in Section~\ref{sec:theory}).


\subsection{Privacy Protocol}
We now turn to our privacy protocol $\cP$ (Algorithm~\ref{alg: p-sdp}), which helps to aggregate data among all agents under privacy constraints. 
The key component of $\cP$ is a \emph{distributed} version of the classic tree-based algorithm, which was originally designed for continual release of private sum statistics~\cite{chan2011private,dwork2010differential}. That is, given a stream of (multivariate) data $\gamma \!=\! (\gamma_1,\ldots, \gamma_K)$, one aims to release $s_k \!=\! \sum_{l=1}^k\gamma_l$ \emph{privately} for all $k \!\in\! [K]$. The tree-based mechanism constructs a complete binary tree $\cT$ in online manner. The leaf nodes contain data $\gamma_1$ to $\gamma_K$, and internal nodes contain the sum of all leaf nodes in its sub-tree, see Fig.~\ref{fig:tree-based} for an illustration. For any new arrival data $\gamma_k$, it only releases a tree node privately, which corresponds to a noisy partial sum (p-sum) between two time indices. As an example, take $k=6$, and hence the new arrival is $\gamma_6$. The tree-based mechanism first computes the p-sum $\sum[5,6] = \gamma_5 + \gamma_6$ (\textcolor{blue}{line}~\ref{line:p-sum} in Algorithm~\ref{alg: p-sdp}). Then, it adds a Gaussian noise with appropriate variance $\sigma_0^2$ to $\sum[5,6]$ and releases the noisy p-sum (\textcolor{blue}{line}~\ref{line:np-sum}). Finally, to compute the prefix sum statistic $\sum[1,6]$ privately, it simply adds noisy p-sums for $\sum[1,4]$ and $\sum[5,6]$, respectively. Reasons behind releasing and aggregating p-sums are that (i) each data point $\gamma_k$ only affects at most $1+\log K$ p-sums (useful for privacy) and (ii) each sum statistic $\sum[1,k]$ only involves at most $1+\log k$ p-sums (useful for utility).

\begin{figure}[ht]
\vspace{-3mm}
\centering
\begin{minipage}[b]{0.6\linewidth}
\centering
\begin{algorithm}[H]
\caption{$\mathcal{P}$, a privacy protocol used in Algorithm~\ref{alg:FedLUCB-SDP}}
\label{alg: p-sdp}
\begin{algorithmic}[1]
%
\STATE {\bfseries Procedure:} Local Randomizer $\mathcal{R}$ at each agent
\STATE{\textcolor{DarkBlue}{\texttt{//Input: stream data $(\gamma_1,\ldots,\gamma_K)$, $\epsilon \!>\! 0, \delta \!\in\! (0,1]$}}}
\begin{ALC@g}
    \FOR{$k\!=\!1, \ldots, K$}
       \STATE Express $k$ in binary form: $k = \sum_j \text{Bin}_j(k) \cdot 2^j$\alglinelabel{line:start}
       \STATE Find index of first one $i_k \!=\! \min\{j: \text{Bin}_j(k) \!=\! 1\}$
       \STATE Compute p-sum $\alpha_{i_k} \!=\! \sum_{j < i_k} \alpha_j \!+\! \gamma_k$\alglinelabel{line:p-sum}
       \STATE Output $\hat{\alpha}_{k} \!=\! \alpha_{i_k} \!+\! \cN(0,\!\sigma_0^2 I)$ \alglinelabel{line:np-sum}
  \ENDFOR
\end{ALC@g}
\STATE {\bfseries Procedure:} Analyzer $\mathcal{A}$ at server
\STATE{\textcolor{DarkBlue}{\texttt{//Input : data from $\cS\!:\!( \hat\alpha_{k,1},\ldots,\hat\alpha_{k,M}), k\! \in\! [K]$ }}}
\begin{ALC@g}
   \FOR{$k\!=\!1, \ldots, K$}
       \STATE Express $k$ in binary and find index of first one $i_k$
        \STATE Add noisy p-sums of all agents: $\widetilde{\alpha}_{i_k} = \sum_{i=1}^M \hat{\alpha}_{k,i}$\alglinelabel{line:add}
        \STATE Output: $\widetilde{s}_k = \sum_{j: \text{Bin}_j(k)=1} \widetilde{\alpha}_j$\alglinelabel{line:out}
  \ENDFOR
\end{ALC@g}
\end{algorithmic}
\end{algorithm}
\end{minipage}
\hspace{0mm}
\begin{minipage}[b]{0.38\linewidth}
\centering
\includegraphics[width=\linewidth]{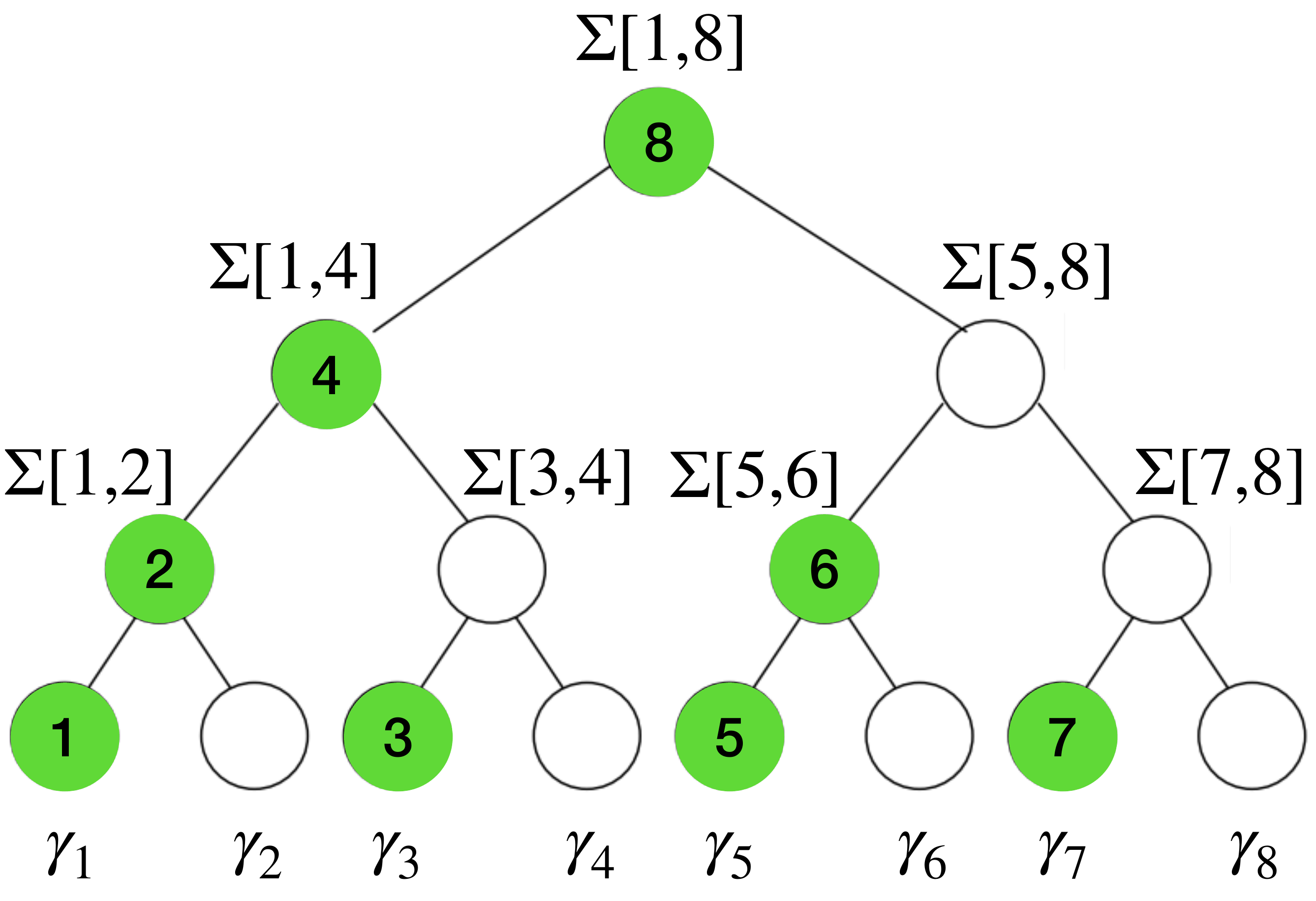}
\caption{\footnotesize{Illustration of the tree-based algorithm. Each leaf node is the stream data and each internal node is a p-sum $\Sigma[i,j] = \sum_{l=i}^j \gamma_l$. The green node corresponds to the newly computed p-sum at each $k$, i.e., $\alpha_{i_k}$ in Algorithm~\ref{alg: p-sdp}.} }
\label{fig:tree-based}
\end{minipage}
\vspace{-3mm}
\end{figure}

Our privacy protocol $\cP \!=\! (\cR,\cS,\cA)$ breaks down the above classic mechanism of releasing and aggregating p-sums into a local randomizer $\cR$ at each agent and an analyzer $\cA$ at the server, separately, while allowing for a possible shuffler in between to amplify privacy. 
For each $k$, the local randomizer $\cR$ at each agent computes and releases the noisy p-sum to a third-party $\cS$ (\textcolor{blue}{lines}~\ref{line:start}-\ref{line:np-sum}). $\cS$ can either be a shuffler that permutes the data uniformly at random (for SDP) or can simply be an identity mapping (for silo-level LDP).  It receives a total of $M$ noisy p-sums, one from each agent, and sends them to the central server. The analyzer $\cA$ at the server first adds these $M$ new noisy p-sums to synchronize them (\textcolor{blue}{line}~\ref{line:add}). It then privately releases the synchronized prefix sum by adding up all relevant synchronized p-sums as discussed in above paragraph  (\textcolor{blue}{line}~\ref{line:out}). 
Finally, we employ $\cP$ to Algorithm~\ref{alg:FedLUCB-SDP} by observing that local data $\gamma_{k,i}$ for batch $k$ and agent $i$ consists of bias vectors $\gamma_{k,i}^{\text{bias}} \!=\! \sum_{t=(k-1)B+1}^{kB} x_{t,i}y_{t,i}$ and covariance matrices $\gamma_{k,i}^{\text{cov}} \!=\! \sum_{t=(k-1)B+1}^{kB} x_{t,i}x_{t,i}^{\top}$, which are stored in $U_i$ and $W_i$ respectively. We denote the randomizer and analyzer for bias vectors as $\cR^{\text{bias}}$ and $\cA^{\text{bias}}$, and for covariance matrices as $\cR^{\text{cov}}$ and $\cA^{\text{cov}}$ in Algorithm~\ref{alg:FedLUCB-SDP}.

\begin{remark}[Sensitivity vs. norm]
Although the $l_2$ norm of each $\gamma_k$ in Algorithm~\ref{alg:FedLUCB-SDP} scales linearly with batch size $B$, its sensitivity is only one, i.e., changing one user's data only changes Euclidean norm of the vector $\gamma_{k,i}^{\text{bias}}$ and Frobenius norm of the matrix $\gamma_{k,i}^{\text{cov}}$ by at most one, due to our boundedness assumption.  It is this sensitivity that determines the noise level for privacy in Algorithm~\ref{alg: p-sdp}.
\end{remark}


\section{Theoretical Results}
\label{sec:theory}
We now show that our generic algorithmic framework (Algorithms~\ref{alg:FedLUCB-SDP} and ~\ref{alg: p-sdp}) enables us to establish regret bounds of federated LCBs under both silo-level LDP and SDP in a simple and unified way. Proofs of all the results are deferred to Appendices~\ref{app:LDP} and \ref{app:SDP} due to space constraint.  
\subsection{Federated LCBs under Silo-level LDP}
\label{sec:main-LDP}
We first present the performance of Algorithm~\ref{alg:FedLUCB-SDP} under silo-level LDP, hence fixing the privacy, regret and communication issues of the state-of-the-art algorithm in~\cite{dubey2020differentially}. The key idea is to inject Gaussian noise with proper variance ($\sigma_0^2$ in Algorithm~\ref{alg: p-sdp}) when releasing a p-sum such that all the released p-sums up to any batch $k \!\in\! [K]$ is $(\epsilon,\delta)$-DP for any agent $i \!\in\! [M]$. Then, by Definition~\ref{def:silo-LDP-general}, it achieves silo-level LDP. Note that in this case, there is no shuffler, which is equivalent to the fact that the third party $\cS$ in $\cP$ is simply an identity mapping, denoted by $\cI$. The following result states this formally.

\begin{theorem}[Performance under silo-level LDP]
\label{thm:LDP}
Fix batch size $B$, privacy budgets $\epsilon \!>\!0$, $\delta \!\in\! (0,1)$. Let $\cP \!=\! (\cR,\cI,\cA)$ be a protocol given by Algorithm~\ref{alg: p-sdp}
with parameters $\sigma_0^2 \!=\! 8 \kappa  \cdot\frac{(\log(2/\delta) + \epsilon)}{\epsilon^2}$, where
$\kappa\!=\! 1\!+\!\log(T/B)$.
Then, under Assumption~\ref{ass:bounded}, Algorithm~\ref{alg:FedLUCB-SDP} instantiated with $\cP$
satisfies $(\epsilon,\delta)$-silo-level LDP. Moreover, for any $\alpha \!\in\! (0,1]$, there exist choices of $\lambda$ and $\{\beta_{t,i}\}_{t,i}$ such that, with probability at least $1-\alpha$, it enjoys a group regret 
\vspace{-1mm}
\begin{align*}
    R_M(T) \!=\! O\left({d M  B}\log T \!+\! d\sqrt{MT}\log(MT/\alpha)\right) \!+\!\widetilde{O}\!\left(\!\sqrt{T}\frac{(Md)^{3/4}\log^{1/4}(1/\delta)}{\sqrt{\epsilon}}\log^{1/4}\!\left(\frac{T}{B\alpha}\!\right)\!\right).
\end{align*}
\end{theorem}
\vspace{-3mm}
The first term in the above regret bound doesn't depend on privacy budgets $\epsilon, \delta$, and serves as a representative regret bound for federated LCBs without privacy constraint. The second term is the dominant one which depends on $\epsilon,\delta$ and denotes the cost of privacy due to injected noise.


\begin{corollary}\label{cor:LDP}
Setting $B \!=\! \sqrt{T/M}$, Algorithm~\ref{alg:FedLUCB-SDP} achieves $\widetilde{O}\left(d\sqrt{MT} + \!\sqrt{T}\frac{(Md)^{3/4}\log^{1/4}(1/\delta)} {\sqrt{\epsilon}}\right)$ group regret, with total $\sqrt{MT}$ synchronizations under $(\epsilon,\delta)$-silo-level LDP. 
\end{corollary}


\textbf{Comparisons with related work.}
First, we avoid privacy leakage and gap in communication analysis of~\cite{dubey2020differentially} by adopting data-independent synchronization. This, however, leads to an $O(\sqrt{T})$ communication cost rather than the reported $O(\log T)$ cost of~\cite{dubey2020differentially}. It remains open to design an data-adaptive communication schedule with a correct performance analysis (see Appendix~\ref{app:priv-com} for more details). 
We also show that privacy cost scales as $O(M^{3/4})$ with number of agents $M$, correcting the reported $\sqrt{M}$ scaling of ~\cite{dubey2020differentially}. Next, we compare our result with that of a (super) single agent running for $MT$ rounds under the central model DP (i.e., where central server is trusted), which serves as a benchmark for our results. As shown in~\cite{shariff2018differentially,pmlr-v162-chowdhury22a}, the total regret for such a single agent is $\widetilde{O}\left(d\sqrt{MT} + \!\sqrt{MT}\frac{d^{3/4}\log^{1/4}(1/\delta)}{\sqrt{\epsilon}}\right)$. Comparing this bound with Corollary~\ref{cor:LDP}, we observe that the privacy cost of federated LCBs under silo-level LDP is a multiplicative $M^{1/4}$ factor higher than a super agent under central DP. This observation motivates us to consider SDP in the next section.

\subsection{Federated LCBs under SDP}
\label{sec:main-SDP}
We now close the above $M^{1/4}$ gap in the privacy cost under silo-level LDP compared to that achieved by a super single agent (with a truseted central server). To do so, we consider federated LCBs under SDP, which still enjoys the nice feature of silo-level LDP that the central server is not trusted. Thanks to our flexible privacy protocol $\cP$, the only change needed compared to silo-level LDP is the introduction of a shuffler $\cS$ to amplify privacy and adjustment of the privacy noise $\sigma_0^2$ accordingly. 
\begin{theorem}[Performance under SDP via amplification]
\label{thm:SDP}
 Fix batch size $B$ and let $\kappa\!=\! 1\!+\!\log(T/B)$. Let $\cP \!=\! (\cR,\cS,\cA)$ be a protocol given by Algorithm~\ref{alg: p-sdp}. Then, under Assumption~\ref{ass:bounded}, there exist constants $C_1,C_2>0$ such that for any 
 $\epsilon \!\le\!  \frac{\sqrt{\kappa}}{C_1T\sqrt{M}}$, $\delta \!\le\!  \frac{\kappa}{C_2 T}$, Algorithm~\ref{alg:FedLUCB-SDP} instantiated with $\cP$ and $\sigma_0^2 \!=\! O\left(\frac{2 \kappa \log(1/\delta)\log(\kappa/(\delta T))\log(M\kappa/\delta)}{\epsilon^2 M}\right)$,
   satisfies $(\epsilon,\delta)$-SDP. Moreover, for any $\alpha \in (0,1]$, there exist choices of $\lambda$ and $\{\beta_{t,i}\}_{t,i}$ such that, with a probability at least $1-\alpha$, it enjoys a group regret 
   \begin{align*}
    R_M(T) \!=\! O\left({d M  B}\log T \!+\! d\sqrt{MT}\log(MT/\alpha)\right) \!+\!\widetilde{O}\!\left(\!d^{3/4}\sqrt{MT}\frac{\log^{3/4}(M\kappa/\delta)}{\sqrt{\epsilon}}\log^{1/4}\!\left(\frac{T}{B\alpha}\!\right)\!\right).
\end{align*}
\end{theorem}

\begin{corollary}
\label{cor:SDP}
Setting $B = \sqrt{T/M}$, Algorithm~\ref{alg:FedLUCB-SDP} achieves  $\widetilde{O}\left(d\sqrt{MT} + \!d^{3/4}\sqrt{MT}\frac{\log^{3/4}(M\kappa/\delta)}{\sqrt{\epsilon}}\right)$ group regret, with total $\sqrt{MT}$ synchronizations under $(\epsilon,\delta)$-SDP.
\end{corollary}

Corollary~\ref{cor:SDP} asserts that privacy cost of federated LCBs under SDP matches that of a super single agent under central DP (up to a log factor in $T, M,\delta$). 

\textbf{Comparison with existing SDP analysis.} A crucial observation here is that the above result doesn't directly follow from existing amplification lemmas. In particular, prior results on privacy amplification ~\cite{feldman2022hiding,erlingsson2019amplification,cheu2019distributed,balle2019privacy} show that shuffling the outputs of $M$ $(\epsilon,\delta)$-LDP algorithms achieve roughly $1/\sqrt{M}$ factor amplification in privacy for small $\epsilon$ -- the key to close the aforementioned gap in privacy cost. However, these amplification results apply \emph{only} when each mechanism is LDP \emph{in the standard sense, i.e., they operate on a dataset of size $n=1$}. This doesn't hold in our case since the dataset at each silo is a stream of $T$ points. \citet{lowy2021private} adopt group privacy to handle the case of $n>1$, which can amplify any general DP mechanism but comes at the expense of a large increase in $\delta$.
To avoid this, we prove a \emph{new amplification lemma} specific to Gaussian DP mechanisms operating on datasets with size $n\!>\!1$. This helps us achieve the required $1/\sqrt{M}$ amplification in $\epsilon$ while keeping the increase in $\delta$ in check. The key idea behind our new lemma is to directly analyze the sensitivity when creating ``clones'' as in~\cite{feldman2022hiding}, but now by accounting for the fact that all $n\!>\!1$ points can be different (see Appendix~\ref{app:SDP} for a formal statement of the lemma).

\subsubsection{SDP guarantee for a wide range of privacy parameters} 

One limitation of attaining SDP via amplification is that the privacy guarantee holds only for small values of $\epsilon,\delta$ (see Theorem~\ref{thm:SDP}). In this section, we propose an alternative privacy protocol to resolute this limitation. This new protocol leverages the same binary tree structure as in Algorithm~\ref{alg: p-sdp} for releasing and aggregating p-sums, but it employs different local randomizers and analyzers for computing (noisy) synchronized p-sums of bias vectors and covariance matrices ($\widetilde{\alpha}_{i_k}$ in Algorithm~\ref{alg: p-sdp}). Specifically, it applies the vector sum mechanism $\cP_{\text{Vec}}$ of \cite{cheu2021shuffle}, which essentially take $n$ vectors as inputs and outputs their noisy sum. Here privacy is ensured by injecting suitable binomial noise to a fixed-point encoding of each vector entry, which depends on $\epsilon,\delta$ and $n$. 

In our case, one cannot directly aggregate $M$ p-sums using $\cP_{\text{Vec}}$ with $n=M$.
This is because each p-sum would then have a large norm ($O(T)$ at the worst case), which would introduce a large amount of privacy noise (cf. Theorem 3.2 in~\cite{cheu2021shuffle}), resulting in worse utility (regret). Instead, we first expand each p-sum resulting in $O(B)$ data points (bias vectors and covariance matrices) each with $O(1)$ norm, where $B$ is the size of each batch. Then, we aggregate all $n=O(BM)$ of those data points using $\cP_{\text{Vec}}$ mechanism (one each for bias vectors and covariance matrices). For example, consider summing bias vectors during batch $k=6$ and refer back to Fig.~\ref{fig:tree-based} for illustration. Here, the p-sum for each agent is given by $\sum[5,6] = \gamma_5 + \gamma_6$ (see \textcolor{blue}{line}~\ref{line:p-sum} in Algorithm~\ref{alg: p-sdp}), the expansion of which results in $2B$ bias vectors ($B$ each for batch 5 and 6). A noisy sum of $n=2BM$ bias vectors is then computed using $\cP_{\text{Vec}}$. We denote the entire mechanism as $\cP^{\cT}_{\text{Vec}}$ -- see Algorithm~\ref{alg: p-vec-T} in Appendix~\ref{app:pvec} for pseudo-code and complete description.

Now, the key intuition behind using $\cP_{\text{Vec}}$ as a building block is that it allows us to compute private vector sums under the shuffle model using nearly the same amount of noise as in the central model. In other words, it ``simulates'' the privacy noise introduced in vector summation under central model using a shuffler. This, in turn, helps us match the regret of a super single agent under central DP while guaranteeing (strictly stronger) SDP. Specifically, we have the same order of regret as in Theorem~\ref{thm:SDP}, but now it holds for a wide range of privacy budgets $\epsilon, \delta$ as presented below formally.

\begin{theorem}[Performance under SDP via vector sum]
\label{thm:SDP-vec}
 Fix batch size $B$ and let $\kappa\!=\! 1\!+\!\log(T/B)$. Let $\cP_{\text{Vec}}^{\cT}$ be a privacy protocol given by Algorithm~\ref{alg: p-vec-T}. Then, under Assumption~\ref{ass:bounded}, there exist parameter choices of $\cP_{\text{Vec}}^{\cT}$ such that for any $\epsilon \!\le\!  60\sqrt{2\kappa \log(2/\delta)}$ and $\delta \! \le \! 1$, Algorithm~\ref{alg:FedLUCB-SDP} instantiated with $\cP_{\text{Vec}}^{\cT}$ satisfies $(\epsilon,\delta)$-SDP. Moreover, for any $\alpha \in (0,1]$, there exist choices of $\lambda$ and $\{\beta_{t,i}\}_{t,i}$ such that, with a probability at least $1-\alpha$, it enjoys a group regret
  \begin{align*}
    R_M(T) \!=\! O\left({d M  B}\log T \!+\! d\sqrt{MT}\log(MT/\alpha)\right) \!+\!\widetilde{O}\!\left(\!d^{3/4}\sqrt{MT}\frac{\log^{3/4}(\kappa d^2/\delta)}{\sqrt{\epsilon}}\log^{1/4}\!\left(\frac{T}{B\alpha}\!\right)\!\right).
\end{align*}
\end{theorem}
\vspace{-1mm}
\begin{remark}[Importance of communicating P-sums]
One of our key techniques behind closing the regret gap under SDP is to communicate and shuffle \emph{only} the p-sums rather than prefix sums. With this we can ensure that each data point (bias vector/covariance matrix) participates only in at most $\log K$ shuffle mechanisms (rather than in $O(K)$ mechanisms if we communicate and shuffle prefix-sums). This helps us to keep the final privacy cost in check after adaptive composition. In other words, one cannot simply use shuffling to amplify privacy of the proposed algorithm in~\cite{dubey2020differentially}  to close the regret gap (even ignoring its privacy and communication issues), since it communicates prefix sums at each synchronization. This again highlights the algorithmic novelty of our privacy protocols (Algorithms~\ref{alg: p-sdp} and~\ref{alg: p-vec-T}), which could be of independent interest. See Appendix~\ref{app:SDP} for further details.
\end{remark}

\subsection{Key Techniques: Overview}
\label{sec:key}
Our first key tool is a generic regret bound for Algorithm~\ref{alg:FedLUCB-SDP} under a mild condition on injected noise. 
Let $t=kB$, and $N_{t,i}, n_{t,i}$ denote
total noise injected up to the $k$-th communication by agent $i$ to covariance matrices $\sum_{s=1}^t x_{s,i}x_{s,i}^{\top}$ and bias vectors $\sum_{s=1}^t x_{s,i}y_{s,i}$, respectively. Moreover, let (i) $\sum_{i=1}^M n_{t,i}$ be a random vector whose entries are independent, mean zero, sub-Gaussian with variance at most $\sigma_1^2$, and
(ii) $\sum_{i=1}^M N_{t,i}$ be a random symmetric matrix whose entries on and above the diagonal are independent sub-Gaussian random variables with variance at most $\sigma_2^2$. Let $\sigma \!=\! \max\lbrace \sigma_1, \sigma_2\rbrace$. Then, we have the following result.
\begin{lemma}[Informal regret bound]
\label{lem:informal}
With high probability, the regret of Algorithm~\ref{alg:FedLUCB-SDP} satisfies
\begin{align*}
     \hspace*{10mm} \!R_{M}(T) \!=\!  \widetilde{O}\left(\!{d M  B} + d\sqrt{MT}+\sqrt{\sigma MT}d^{3/4} \!\right).
\end{align*}
\end{lemma}
Armed with the above lemma, one only needs to determine the noise variance $\sigma^2$ under different privacy constraints. For silo-level LDP, we use concentrated differential privacy~\cite{bun2016concentrated} to obtain a tighter privacy accounting. In this process, we also utilize the nice properties of the tree-based mechanism. The final total noise level is $\sigma^2 =  8 M \kappa^2  \cdot\frac{(\log(2/\delta) + \epsilon)}{\epsilon^2}$ with $\kappa:= 1+\log K$. For SDP, the key idea is to leverage privacy amplification of $1/\sqrt{M}$ by shuffling. Hence, the noise variance by each agent is roughly $1/M$ of the noise under LDP. By Lemma~\ref{lem:informal}, we thus shave an $M^{1/4}$ factor from the regret under silo-level LDP. However, as mentioned before, the key technique to achieve this is our new amplification lemma in Appendix~\ref{app:SDP}.  For SDP via vector sum, we utilize the property of $\cP_{\text{Vec}}$ to compute each noisy synchronized p-sum under SDP using noise level $\widetilde{O}(\kappa/\epsilon^2)$ (where we use the fact that each data point only participates at most $\kappa$ times and advanced composition). Then, by the binary tree structure again, each private prefix sum only requires at most $\kappa$ noisy synchronized p-sums. Thus, the total amount of noise is $\widetilde{O}(\kappa^2/\epsilon^2)$.

\section{Simulation Results}\label{sec:exp}

We evaluate regret performance of Algorithm~\ref{alg:FedLUCB-SDP} under silo-level LDP and SDP, which we abbreviate as LDP-FedLinUCB and SDP-FedLinUCB, respectively. We fix confidence level $\alpha \!=\! 0.01$, batchsize $B\!=\!25$ and study comparative performances under varying privacy budgets $\epsilon,\delta$.
We plot time-averaged group regret $\text{Reg}_M(T)/T$ in Figure~\ref{fig:all_algos} by averaging results over 25 parallel runs. Our simulations are proof-of-concept only; we do not tune any hyperparameters.

\textbf{Synthetic bandit instance.}
We simulate a LCB instance with a parameter $\theta^*$ of dimension $d=10$ and $|\cK_i| = 100$ actions for each of the $M$ agents. Similar to \citet{vaswani2020old}, 
we generate $\theta^*$ and feature vectors by sampling a $(d\!-\!1)$-dimensional vectors of norm $1/\sqrt{2}$ uniformly at random, and append it with a $1/\sqrt{2}$ entry. 
Rewards are corrupted with Gaussian $\cN(0,0.25)$ noise.

\textbf{Real-data bandit instance.}
We generate bandit instances from Microsoft Learning to Rank dataset \citep{DBLP:journals/corr/QinL13}. Queries form the contexts
$c$ and actions $a$ are the available documents. The dataset contains 10K queries, each with
up to 908 judged documents, where the query-document pairs are judged on a 3-point scale,
$\text{rel}(c, a) \in \lbrace 0,1,2 \rbrace$. Each pair $(c,a)$ has a feature vector $\phi(c,a)$, which is partitioned into title and body features of dimensions 57 and 78, respectively. We first train a lasso regression model on title features to predict relevances from $\phi$, and take this model as the bandit parameter $\theta^*$ with $d=57$ (similar experiment with body features is reported in Appendix~\ref{app:sim}). Next, we divide the queries equally into $M\!=\!10$ agents and assign corresponding feature vectors to the agents. This way, we obtain a federated LCB instance with $10$ agents, each with number of actions $|\cK_i| \le 908$.

\textbf{Observations.} In sub-figure (a), we compare performance of LDP-FedLinUCB and SDP-FedLinUCB (with amplification based privacy protocol $\cP$) on synthetic Gaussian bandit instance with $M\!=\!100$ agents under privacy budget $\delta\!=\!0.0001$ and $\epsilon \!=\! 0.001$ or $0.0001$. We observe that regret of SDP-FedLinUCB is less than LDP-FedLinUCB for both values of $\epsilon$, which is consistent with our theoretical results. Here, we only work with small privacy budgets since the privacy guarantee of Theorem~\ref{thm:SDP} holds for $\epsilon,\delta \!\ll\! 1$.
Instead, in sub-figure (b), we consider higher privacy budgets as suggested in Theorem~\ref{thm:SDP-vec} (e.g. $\epsilon\!=\!0.2$, $\delta\!=\!0.1$) and
compare the regret performance of LDP-FedLinUCB and SDP-FedLinUCB (with vecor-sum based privacy protocol $\cP_{\text{vec}}^\cT$). As expected, here also we observe that regret of SDP-FedLinUCB decreases faster than that of LDP-FedLinUCB.

Next, we benchmark the performance of Algorithm~\ref{alg:FedLUCB-SDP} under silo-level LDP (i.e. LDP-FedLinUCB) against a non-private Federated LCB algorithm with fixed communication schedule, which we build upon the algorithm of~\citet{abbasi2011improved} and refer as FedLinUCB. In sub-figure (c), we demonstrate the cost of privacy under silo-level LDP on real-data bandit instance by varying $\epsilon$ in the set $\lbrace 0.2,1,5\rbrace$ while keeping $\delta$ fixed to 0.1. We observe that regret of LDP-FedLinUCB decreases and comes closer to that of FedLinUCB as $\epsilon$ increases (i.e., level of privacy protection decreases).
A similar regret behavior is noticed under SDP also (postponed to Appendix~\ref{app:sim}).


\begin{figure}[ht!]
\begin{subfigure}[t]{.33\linewidth}
            \centering
			\includegraphics[width = 1.9in]{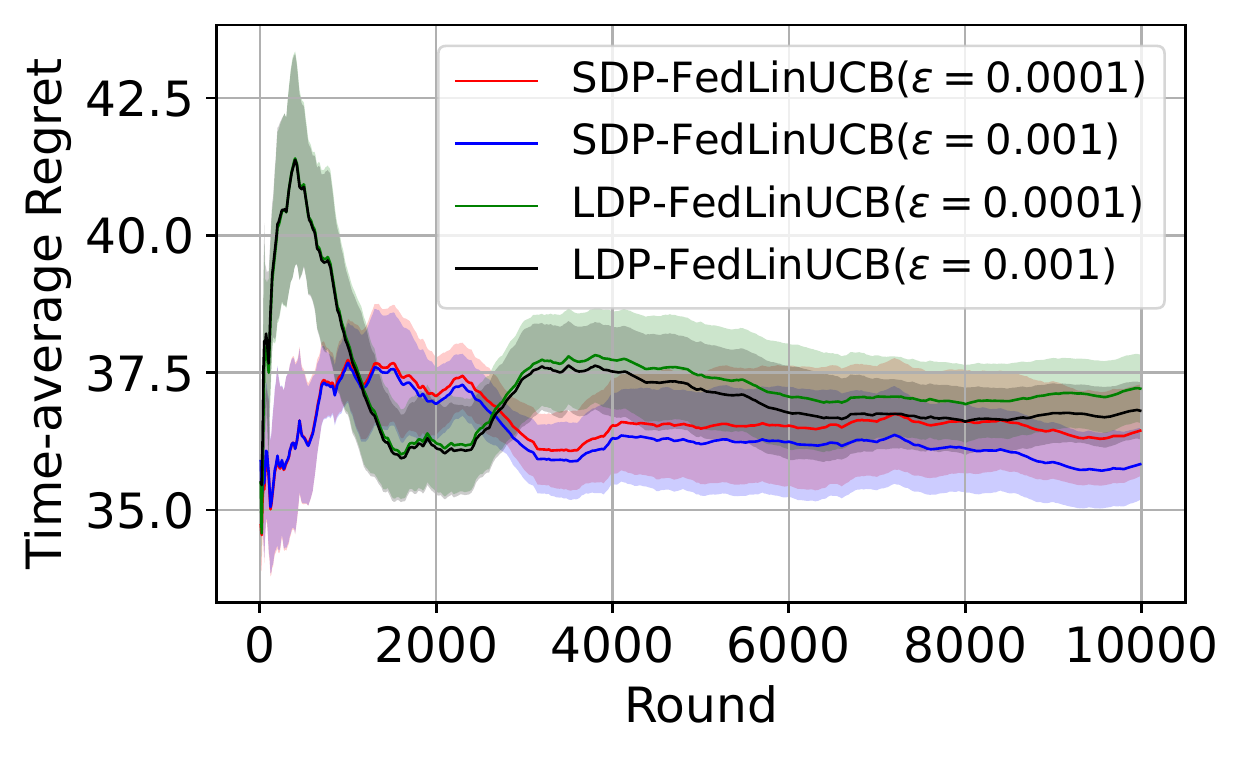}
   \vskip -1.5mm
			\caption{Synthetic data ($M=100$)}
		\end{subfigure}\ \
  \begin{subfigure}[t]{.33\linewidth}
        \centering
		\includegraphics[width = 1.85in]{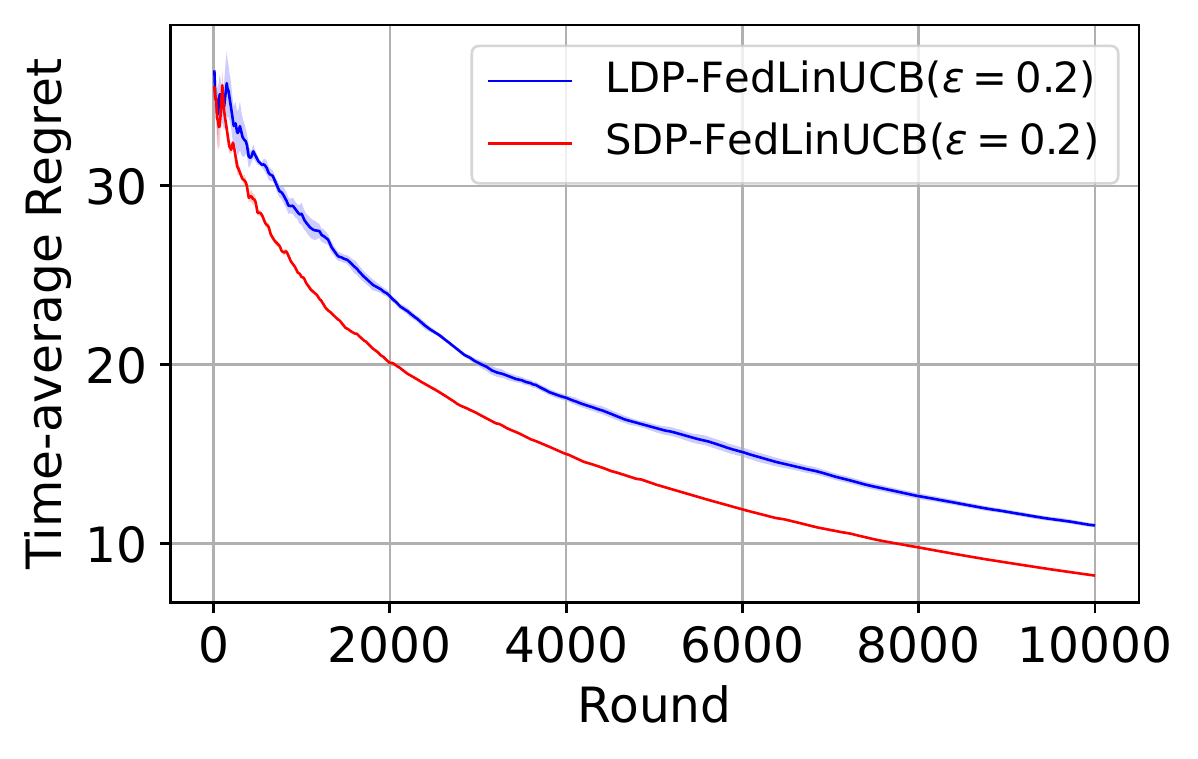}
  \vskip -1.5mm
			\caption{Synthetic data ($M=100$) }
		\end{subfigure}\ \
  \begin{subfigure}[t]{.32\linewidth}
            \centering
			\includegraphics[width = 1.8in]{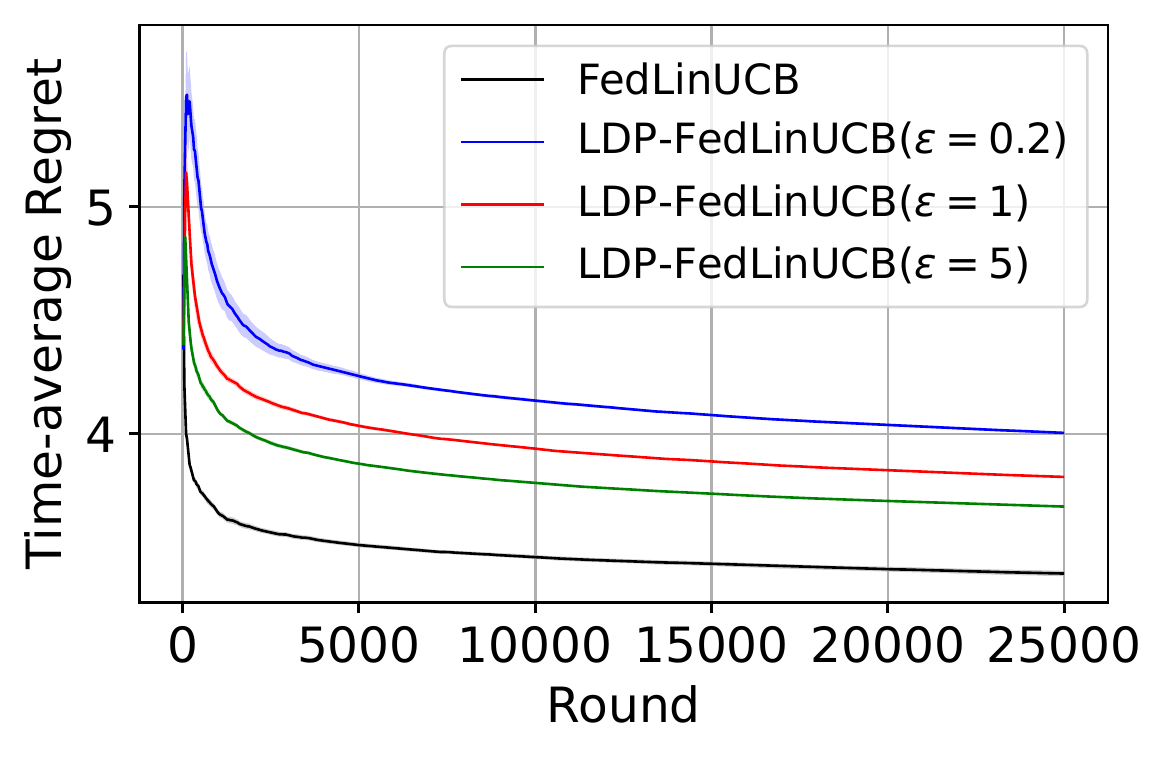}
   \vskip -1.5mm
			\caption{Real data ($M=10$)}
		\end{subfigure}
\caption{\footnotesize{Comparison of time-average group regret for LDP-FedLinUCB (silo-level LDP), SDP-FedLinUCB (shuffle model) and FedLinUCB (non-private) under varying privacy budgets $\epsilon, \delta$ on (a, b) synthetic Gaussian bandit instance and (c) bandit instance generated from MSLR-WEB10K Learning to Rank dataset. }  }\label{fig:all_algos}
		
\end{figure}
\section{Concluding Remarks}
\label{sec:con}

\noindent\textbf{Silo-level LDP/SDP vs. other privacy notions.} It is helpful to compare silo-level LDP and SDP with other existing privacy notions. In Appendix~\ref{app:ldp}, we compare it with the standard local model, central model and shuffle model for single-agent LCBs. As a by-product, via a simple tweak of Algorithm~\ref{alg:FedLUCB-SDP}, we also show how to achieve a slightly stronger privacy guarantee than silo-level LDP in the sense that now the action selection is also based on private data only. With this, we can not only protect against colluding among other silos (as in silo-level LDP), but against colluding among users within the same silo (as in standard central DP). 

\noindent\textbf{Non-unique users.} In this theoretical work, we assume that all $MT$ users are unique. In practice, it is often the case that the same user can participate in multiple rounds within the same silo or across different silos. For example, the same patient can have multiple medical tests at the same hospital or across different types of hospitals. We discuss how to handle the case of non-unique users in Appendix~\ref{app:non-unique}.

\noindent\textbf{Future work.} One immediate future work is to reduce communication costs by overcoming the challenges in private adaptive communication, see Appendix~\ref{app:priv-com}. Another direction could be considering similar cross-silo federated learning for private RL, especially with linear function approximation (cf.~\cite{ZhouXingyuRL}.)

\section{Acknowledgements}
XZ is supported in part by NSF CNS-2153220. XZ would like to thank Abhimanyu Dubey for discussions on the work~\cite{dubey2020differentially}. 
XZ would also like to thank Andrew Lowy and Ziyu Liu for insightful discussions on the privacy notion for cross-silo federated learning. XZ would also thank Vitaly Feldman and Audra McMillan for the discussion on some subtleties behind ``hiding among the clones''.


\printbibliography
\newpage

\appendix
\section{More Discussions on Gaps in SOTA}

\label{app:gap}
In this section, we provide more details on the current gaps in~\cite{dubey2020differentially}, especially on privacy violation and communication cost. It turns out that both gaps come from the fact that an adaptive communication schedule is employed in~\cite{dubey2020differentially}.
\subsection{More on violation of silo-level LDP}
As shown in the main paper, Algorithm 1 in~\cite{dubey2020differentially} does not satisfy silo-level LDP. 
To give a more concrete illustration of privacy leakage, we now specify the form of $f$\footnote{There is some minor issue in the form of $f$ in~\cite{dubey2020differentially}. The correct one is given by our restatement of their Algorithm 1, see line 9 in Algorithm~\ref{alg:FedLUCB-dubey}.}, local data $X_i$ and synchronized data $Z$ in~\eqref{eq:sync} according to~\cite{dubey2020differentially}. In particular, a communication is triggered  at round $t$ if for any silo $i$, it holds that
\begin{align}
\label{eq:sync-app}
\!(t\!-\!t')
\log\left[\frac{\det\left(Z \!+\! \sum_{s=t'+1}^t x_{s,i}x_{s,i}^{\top} \!+\! \lambda_{\min} I\right)}{\det\left(Z \!+\! \lambda_{\min} I\right)}\right] \!>\! D,\!
\end{align}
where $t'$ is the latest synchronization time before $t$, $Z$ is all synchronized (private) covariance matrices up to time $t'$, $\lambda_{\min} > 0$ is some regularization constant (which depends on privacy budgets $\epsilon,\delta$) and $D > 0$ is some suitable threshold (which depends on number of silos $M$). 

With the above explicit form in hand, we can give a more concrete discussion of Example~\ref{ex:leak}. A communication is triggered at round $t=1$ if  $\det\left( x_{1,m}x_{1,m}^{\top} \!+\! \lambda_{\min} I\right) > \det\left(\lambda_{\min} I\right) e^D$ holds for any silo $m$. This implies that $(\lambda_{\min} + \norm{x_{1,m}}^2) \lambda_{\min}^{d-1} > e^D \lambda_{\min}^d$, which, in turn, yields $\norm{x_{1,m}}^2 > \lambda_{\min}(e^D-1)=:C$. Now, if $\norm{x_{1,j}}^2 \le C$, then silo $j$ immediately knows that $\norm{x_{1,i}}^2 > C$, where $C$ is a known constant. Since $x_{1,i}$ contains the context information of the user (Alice), this norm condition could immediately reveal that some specific features in the context vector are active (e.g., Alice has both diabetes and heart disease), thus leaking Alice's private and sensitive information to silo $j$.

\begin{remark}
\label{rem:weak-FedDP}
The above result has two implications: (i) the current proof strategy for Fed-DP guarantee in~\cite{dubey2020differentially} does not hold since it essentially relies on the post-processing of DP through silo-level LDP; (ii) Fed-DP could fail to handle reasonable adversary model in cross-silo federated LCBs. That is, even if Algorithm 1 in~\cite{dubey2020differentially} satisfies Fed-DP, it still cannot protect Alice's information from being inferred by a malicious silo (which is a typical adversary model in cross-silo FL). Thus, we believe that silo-level LDP is a more proper privacy notion for cross-silo federated LCBs. 
\end{remark}

\subsection{More on violation of Fed-DP}
As shown in the main paper, Algorithm 1 in~\cite{dubey2020differentially} also does not satisfy its weaker notion of Fed-DP. To give a more concrete illustration, recall Example~\ref{ex:leak} and let us define $m_{i,j}$ as the message/data sent from silo $i$ to silo $j$ after round $t=1$. Suppose in the case of Alice, there is no synchronization and hence $m_{i,j} = 0$. On the other hand, in the case of Tracy (i.e., the first user at silo $i$ changes from Alice to Tracy),  suppose synchronization is triggered by silo $i$ via rule~\eqref{eq:sync} due to Tracy's data. Then, according to~\cite{dubey2020differentially}, $m_{i,j} = x_{1,i}y_{1,i} + \cN$ (consider bias vector here), where $\cN$ is the injected noise when silo $i$ sends out its data. Now, based on the requirement of Fed-DP, the recommended action at silo $j$ in round $t=2$ needs to be ``similar'' or ``indistinguishable'' in probability under the change from Alice to Tracy. Note that silo $j$ chooses its action at round $t=2$ based on its local data (which is unchanged) and $m_{i,j}$, via \emph{deterministic} selection rule (i.e., LinUCB) in Algorithm 1 of~\cite{dubey2020differentially}. Thus, Fed-DP essentially requires $m_{i,j}$ to be close in probability when Alice changes to Tracy, which is definitely not the case (i.e., $0$ vs. $x_{1,i}y_{1,i} + \cN$). Thus, Algorithm 1 in~\cite{dubey2020differentially} also fails Fed-DP. 

\begin{remark}
One can also think from the following perspective: the non-private data-dependent sync rule (i.e.,~\eqref{eq:sync-app}) in~\cite{dubey2020differentially} impacts the communicated messages/data as well, which cannot be made private by injecting noise when sending out data. To rescue, a possible approach is to use \emph{private (noisy)} data in rule~\eqref{eq:sync-app} when determining synchronization (while still injecting noise when sending out data). As a result, whether there exists a synchronization would be ``indistinguishable'' under Alice or Tracy and hence $m_{i,j}$ now would be similar. However, this approach still suffers the gap in communication cost analysis (see below) and moreover it will incur new challenges in regret analysis, see Appendix~\ref{app:priv-com} for a detailed discussion on this approach.  
\end{remark}

\subsection{More on communication cost analysis}
The current analysis in~\cite{dubey2020differentially} (cf. Proposition 5) for communication cost (i.e., how many rounds of communication within $T$) essentially follows the approach in the non-private work~\cite{wang2019distributed} (cf. proof of Theorem 4). However, due to additional privacy noise injected into the communicated data, one key step of the approach in~\cite{wang2019distributed} fails in the private case. In the following, we first point out the issue using notations in~\cite{dubey2020differentially}. 


The key issue in its current proof of Proposition 5 in~\cite{dubey2020differentially} is that
\begin{align}
\label{eq:com-gap}
   \log \frac{\det(\mathbf{S}_{i,t+n'})}{\det(\mathbf{S}_{i,t})} > \frac{D}{n'} 
\end{align}
which appears right above Eq. 4 in~\cite{dubey2020differentially} does not hold. More specifically, $[t, t+n']$ is the $i$-th interval between two communication steps and $\mathbf{S}_{i,t}, \mathbf{S}_{i,t+n'}$ are corresponding synchronized private matrices. At the time $t+n'$, we know~\eqref{eq:sync-app} is satisfied by some silo (say $j \in [M]$), since there is a new synchronization. In the non-private case, $\mathbf{S}_{i,t+n'}$ simply includes some additional local covariance matrices from silos other than $j$, which are positive semi-definite (PSD). As a result,~\eqref{eq:com-gap} holds. However, in the private case, $\mathbf{S}_{i,t+n'}$ includes the \emph{private} messages from silos other than $j$, which may not be positive semi-definite (PSD), since there are some new covariance matrices as well as \emph{new Gaussian privacy noise} (which could be negative definite). Thus,~\eqref{eq:com-gap} may not hold anymore.

\section{A Generic Regret Analysis for Algorithm~\ref{alg:FedLUCB-SDP}}
\label{app:general}

In this section, we formally establish Lemma~\ref{lem:informal}, i.e., our generic regret bound of Algorithm~\ref{alg:FedLUCB-SDP} under sub-Gaussian noise condition. To this end, let us first recall the following notations. Fix $B, T \in \mathbb{N}$, we let $K = T/B$ be the total number of communication steps. For all $i\in [M]$ and all $t = k B$, $k\in [K]$, we let  $N_{t,i} = \widetilde{W}_{t,i} - \sum_{s=1}^t x_{s,i}x_{s,i}^{\top} $ and $ n_{t,i} = \widetilde{U}_{t,i} - \sum_{s=1}^t x_{s,i}y_{s,i}$ be the cumulative injected noise up to the $k$-th communication by agent $i$. We further let $H_{t}:= \lambda I_d + \sum_{i \in [M]} N_{t,i}$ and $h_{t}:=  \sum_{i \in [M]} n_{t,i}$.

\begin{assumption}[Regularity]
\label{ass:reg}
    Fix any $\alpha \in (0,1]$, with probability at least $1-\alpha$, we have $H_t$ is positive definite and  there exist constants $\lambda_{\text{max}}, \lambda_{\text{min}}$ and $\nu$ depending on $\alpha$ such that for all $t=kB$, $k\in [K]$
    \begin{align*}
    \norm{H_{t}} \le \lambda_{\max},\quad \norm{H_{t}^{-1}} \le 1/\lambda_{\min}, \quad \norm{h_{t}}_{H_{t}^{-1}} \le \nu.
    \end{align*} 
\end{assumption}

With the above regularity assumption and the boundedness in Assumption~\ref{ass:bounded}, we fist establish the following general regret bound of Algorithm~\ref{alg:FedLUCB-SDP}, which can be viewed as a direct generalization of the results in~\cite{shariff2018differentially,pmlr-v162-chowdhury22a} to the federated case.

\begin{lemma}
\label{lem:general}
Let Assumptions~\ref{ass:reg} and~\ref{ass:bounded} hold. Fix any $\alpha \in (0,1]$, there exist choices of $\lambda$ and $\{\beta_{t,i}\}_{t\in[T],i\in [M]}$ such that, with probability at least $1-\alpha$,  the group regret of Algorithm~\ref{alg:FedLUCB-SDP} satisfies 
\begin{align*}
    \text{Reg}_M(T) = O\left(\beta_{T} \sqrt{dMT \log\left(1+\frac{MT}{d\lambda_{\min}}\right)} \right) + O\left(  M \cdot B \cdot {d}\log\left(1+\frac{MT}{d\lambda_{\min}}\right)\right),
\end{align*}
where 
$\beta_{T}:=\sqrt{2\log\left(\frac{2}{\alpha}\right) + d\log\left(1+\frac{MT}{d\lambda_{\min}}\right)} + \sqrt{\lambda_{\max}} + \nu.$
\end{lemma}
Lemma~\ref{lem:informal} is a corollary of the above result, which holds by bounding $\lambda_{\max}, \lambda_{\min}, \nu$ under sub-Gaussian privacy noise. 
\begin{assumption}[sub-Gaussian private noise]
\label{ass:subG}
There exist constants $\widetilde{\sigma}_1$ and $\widetilde{\sigma}_2$ such that for all $t = kB$, $k \in [K]$:  (i) $\sum_{i=1}^M n_{t,i}$ is a random vector whose entries are independent, mean zero, sub-Gaussian with variance at most $\widetilde{\sigma}_1^2$, and
(ii) $\sum_{i=1}^M N_{t,i}$ is a random symmetric matrix whose entries on and above the diagonal are independent sub-Gaussian random variables with variance at most $\widetilde{\sigma}_2^2$. Let $\sigma^2 \!=\! \max\lbrace \widetilde \sigma_1^2,\widetilde \sigma_2^2\rbrace$.
\end{assumption}
Now, we are ready to state the formal version of Lemma~\ref{lem:informal} as follows. 
\begin{lemma}[Formal statement of Lemma~\ref{lem:informal}]
\label{lem:subG}
Let Assumptions~\ref{ass:subG} and~\ref{ass:bounded} hold. Fix time horizon $T \in \Nat$, batch size $B \in [T]$, confidence level $\alpha \in (0,1]$. Set $\lambda = \Theta(\max\{1, \sigma (\sqrt{d}+ \sqrt{\log(T/(B\alpha))}\})$ and $\beta_{t,i} = \sqrt{2\log\left(\frac{2}{\alpha}\right) + d\log\left(1+\frac{Mt}{d\lambda}\right)} + \sqrt{\lambda} $  for all $i \in [M]$. Then, Algorithm~\ref{alg:FedLUCB-SDP} achieves group regret 
\begin{align*}
    \text{Reg}_M(T) = O\left({d M  B}\log T + d\sqrt{MT}\log(MT/\alpha)\right) + O\left(\sqrt{\sigma MT\log(MT)}d^{3/4} \log^{1/4} (T/(B\alpha))\right)
\end{align*}
with probability at least $1-\alpha$.
\end{lemma}

\subsection{Proofs}

\begin{proof}[Proof of Lemma~\ref{lem:general}]
We divide the proof into the following six steps. Let $\cE$ be the event given in Assumption~\ref{ass:reg}, which holds with probability at least $1-\alpha$ under Assumption~\ref{ass:reg}. In the following, we condition on the event $\cE$.

\textbf{Step 1: Concentration.} In this step, we will show that with high probability, $\norm{\theta^* - \hat{\theta}_{t,i}}_{V_{t,i}} \le \beta_{t,i}$ for all $i \in [M]$. Fix an agent $i \in [M]$ and $t \in [T]$, let $t_{\text{last}}$ be the latest communication round of all agents before $t$. By the update rule, we have 
\begin{align*}
    \hat{\theta}_{t,i} &= V_{t,i}^{-1}(\widetilde{U}_{\text{syn}} + U_{i}) \\
    &= V_{t,i}^{-1}\left(\sum_{j=1}^M\sum_{s=1}^{t_{\text{last}}}x_{s,j}y_{s,j} + \sum_{j=1}^M n_{t_{\text{last}},j} + \sum_{s=t_{\text{last}}+1}^{t-1}x_{s,i}y_{s,i}\right)\\
    &=\left(\lambda I + \sum_{j=1}^M\sum_{s=1}^{t_{\text{last}}}x_{s,j}x_{s,j}^{\top} + \sum_{j=1}^M N_{t_{\text{last}},j} + \sum_{s=t_{\text{last}}+1}^{t-1}x_{s,i}x_{s,i}^{\top} \right)^{-1} \left(\sum_{j=1}^M\sum_{s=1}^{t_{\text{last}}}x_{s,j}y_{s,j} + \sum_{j=1}^M n_{t_{\text{last}},j} + \sum_{s=t_{\text{last}}+1}^{t-1}x_{s,i}y_{s,i}\right).
\end{align*}
By the linear reward function $y_{s,j} = \inner{x_{s,j}}{\theta^*} + \eta_{s,j}$ for all $j\in [M]$ and elementary algebra, we have 
\begin{align*}
    \theta^* - \hat{\theta}_{t,i} = V_{t,i}^{-1}\left(H_{t_{\text{last}}} \theta^* - \sum_{j=1}^M\sum_{s=1}^{t_{\text{last}}}x_{s,j}\eta_{s,j} - \sum_{s=t_{\text{last}}+1}^{t-1}x_{s,i}\eta_{s,i} - h_{t_{\text{last}}}\right),
\end{align*}
where we recall that $H_{t_{\text{last}}} = \lambda I + \sum_{j=1}^M N_{t_{\text{last}},j}$ and $h_{t_{\text{last}}} = \sum_{j=1}^M n_{t_{\text{last}},j}$.

Thus, multiplying both sides by $V_{t,i}^{1/2}$, yields
\begin{align*}
    \norm{\theta^* - \hat{\theta}_{t,i}}_{V_{t,i}} &\le \norm{\sum_{j=1}^M\sum_{s=1}^{t_{\text{last}}}x_{s,j}\eta_{s,j} + \sum_{s=t_{\text{last}}+1}^{t-1}x_{s,i}\eta_{s,i}}_{V_{t,i}^{-1}} + \norm{H_{t_{\text{last}}}\theta^*}_{V_{t,i}^{-1}} + \norm{h_{t_{\text{last}}}}_{V_{t,i}^{-1}} \\
    &\lep{a} \norm{\sum_{j=1}^M\sum_{s=1}^{t_{\text{last}}}x_{s,j}\eta_{s,j} + \sum_{s=t_{\text{last}}+1}^{t-1}x_{s,i}\eta_{s,i}}_{(G_{t,i} + \lambda_{\text{min}} I)^{-1}} + \norm{\theta^*}_{H_{t_{\text{last}}}} + \norm{h_{t_{\text{last}}}}_{H_{t_{\text{last}}}^{-1}}\\
     &\lep{b} \norm{\sum_{j=1}^M\sum_{s=1}^{t_{\text{last}}}x_{s,j}\eta_{s,j} + \sum_{s=t_{\text{last}}+1}^{t-1}x_{s,i}\eta_{s,i}}_{(G_{t,i} + \lambda_{\text{min}} I)^{-1}} + \sqrt{\lambda_{\text{max}}} + \nu
\end{align*}
where (a) holds by $V_{t,i} \succeq H_{t_{\text{last}}}$ and $V_{t,i} \succeq G_{t,i} + \lambda_{\text{min}} I$ with $G_{t,i} := \sum_{j=1}^M\sum_{s=1}^{t_{\text{last}}}x_{s,j}x_{s,j}^{\top} + \sum_{s=t_{\text{last}}+1}^{t-1}x_{s,i}x_{s,i}^{\top}$ (i.e., non-private Gram matrix) under event $\cE$; (b) holds by the boundedness of $\theta^*$ and event $\cE$. 

For the remaining first term, we can use self-normalized inequality (cf. Theorem 1 in~\cite{abbasi2011improved}) with a proper filtration\footnote{In particular, by the i.i.d noise assumption across time and agents, one can simply construct the filtration sequentially across agents and rounds, which enlarges the single-agent filtration by a factor of $M$.}. In particular, we have for any $\alpha \in (0,1]$, with probability at least $1-\alpha$, for all $t \in [T]$
\begin{align*}
\norm{\sum_{j=1}^M\sum_{s=1}^{t_{\text{last}}}x_{s,j}\eta_{s,j} + \sum_{s=t_{\text{last}}+1}^{t-1}x_{s,i}\eta_{s,i}}_{(G_{t,i} + \lambda_{\text{min}} I)^{-1}}  \le \sqrt{2\log\left(\frac{1}{\alpha}\right) + \log\left(\frac{\det(G_{t,i} + \lambda_{\min}I) }{\det(\lambda_{\min} I)}\right)}.
\end{align*}
Now, using the trace-determinant lemma (cf. Lemma 10 in~\cite{abbasi2011improved}) and the boundedness condition on $\norm{x_{s,j}}$ for all $s\in [T]$ and $j \in [M]$, we have 
\begin{align*}
    \det(G_{t,i} + \lambda_{\min}I) \le \left(\lambda_{\min} + \frac{Mt }{d}\right)^d.
\end{align*}
Putting everything together, we have with probability at least $1-2\alpha$, for all $i \in [M]$ and all $t\in [T]$, $\norm{\theta^* - \hat{\theta}_m}_{V_{t,i}} \le \beta_{t,i} = \beta_t$, where 
\begin{align}
\label{eq:beta}
    \beta_{t}:=\sqrt{2\log\left(\frac{1}{\alpha}\right) + d\log\left(1+\frac{Mt}{d\lambda_{\min}}\right)} + \sqrt{\lambda_{\max}} + \nu.
\end{align}
\textbf{Step 2: Per-step regret.} With the above concentration result, based on our UCB policy for choosing the action, we have the classic bound on the per-step regret $r_{t,i}$, that is, with probability at least $1-2\alpha$  
\begin{align*}
    r_{t,i} &= \inner{\theta^*}{x_{t,i}^*} - \inner{\theta^*}{x_{t,i}}\nonumber\\
    &\ep{a} \inner{\theta^*}{x_{t,i}^*} - \text{UCB}_{t,i}(x_{t,i}^*) + \text{UCB}_{t,i}(x_{t,i}^*) - \text{UCB}_{t,i}(x_{t,i}) + \text{UCB}_{t,i}(x_{t,i})- \inner{\theta^*}{x_{t,i}} \\
    & \lep{b} 0 + 0 + 2 \beta_{t,i} \norm{x_{t,i}}_{V_{t,i}^{-1}}  \le 2 \beta_{T} \norm{x_{t,i}}_{V_{t,i}^{-1}}
\end{align*}
where in (a), we let $\text{UCB}_{t,i}(x) := \inner{\hat{\theta}_{t,i}}{x} + \beta_{t,i} \norm{x}_{V_{t,i}^{-1}}$; (b) holds by the optimistic fact of UCB (from the concentration), greedy action selection, and the concentration result again. 

\textbf{Step 3: Regret decomposition by good and bad epochs.} In Algorithm~\ref{alg:FedLUCB-SDP}, at the end of each synchronization time $t = k B$ for $k \in [K]$, all the agents will communicate with the server by uploading private statistics and downloading the aggregated ones from the server. We then divide time horizon $T$ into epochs by the communication (sync) rounds. In particular, the $k$-th epoch contains rounds between $(t_{k-1},t_k]$, where $t_k = kB$ is the $k$-th sync round. We define $V_k := \lambda_{\text{min}} I + \sum_{i=1}^M \sum_{t=1}^{t_k} x_{t,i}x_{t,i}^{\top}$, i.e., all the data at the end of the $k$-th communication plus a regularizer. Then, we say that the $k$-th epoch is a ``good'' epoch if $\frac{\det(V_k)}{\det(V_{k-1})} \le 2$; otherwise it is a ``bad'' epoch. 
Thus, we can divide the group regret into two terms: $$
\text{Reg}_{M}(T) = \sum_{i \in [M]} \sum_{t \in\text{good epochs} } r_{t,i} + \sum_{i\in [M]} \sum_{t \in \text{bad epochs} } r_{t,i}.$$

\textbf{Step 4: Bound the regret in good epochs.} To this end, we introduce an \emph{imaginary} single agent that pulls all the $MT$ actions in the following order:  $x_{1,1,}, x_{1,2}, \ldots, x_{1,M}, x_{2,1}, \ldots, x_{2,M}, \ldots, x_{T,1}, \ldots, x_{T,M}$. We define a corresponding \emph{imaginary} design matrix $\bar{V}_{t,i} = \lambda_{\text{min}} I + \sum_{p<t, q\in [M]} x_{p,q}x_{p,q}^{\top}  + \sum_{p=t, q <i} x_{p,q}x_{p,q}^{\top}$, i.e., the design matrix right \emph{before}  $x_{t,i}$. The key reason behind this construction is that one can now use the standard result (i.e., the elliptical potential lemma (cf. Lemma 11 in~\cite{abbasi2011improved})) to bound the summation of bonus terms, i.e., $\sum_{t,i} \norm{x_{t,i}}_{\bar{V}_{t,i}^{-1}}$. 

Suppose that $t\in [T]$ is within the $k$-th epoch. One {key property} we will use is that for all $i$, $V_{k} \succeq \bar{V}_{t,i}$ and $G_{t,i} + \lambda_{\min} I \succeq V_{k-1}$, which simply holds by their definitions. This property enables us to see that for any $t \in \text{good epochs}$, $\det(\bar{V}_{t,i}) / \det(G_{t,i} + \lambda_{\min} I) \le 2$. This is important since by the standard ``determinant trick'', we have
\begin{align}
\label{eq:det-trick}
   \norm{x_{t,i}}_{(G_{t,i} + \lambda_{\min} I)^{-1}} \le \sqrt{2}\norm{x_{t,i}}_{\bar{V}_{t,i}^{-1}}.
\end{align}
In particular, this follows from Lemma 12 in~\cite{abbasi2011improved}, that is, for two positive definite matrices  $A, B\in \mathbb{R}^{d \times d}$ satisfying $A \succeq B$, then for any $x \in \mathbb{R}^d$, $\norm{x}_{A} \le \norm{x}_B \cdot \sqrt{\det(A)/\det(B)}$. Note that here we also use $\det(A) = 1/\det(A^{-1})$.
Hence, we can bound the regret in good epochs as follows. 
\begin{align}
\label{eq:good-regret}
    \sum_{i \in [M]} \sum_{t \in \text{good epochs} } r_{t,i} &\lep{a} \sum_{i \in [M]} \sum_{t \in \text{good epochs}}  \min\{ 2\beta_T\norm{x_{t,i}}_{V_{t,i}^{-1}},1\}\nonumber\\
    &\lep{b}\sum_{i \in [M]} \sum_{t \in \text{good epochs}}  \min\{ 2\beta_T\norm{x_{t,i}}_{(G_{t,i} + \lambda_{\min} I)^{-1}},1\} \nonumber\\
    & \lep{c} \sum_{i \in [M]} \sum_{t \in \text{good epochs}}  \min\{ 2\sqrt{2} \beta_{T}\norm{x_{t,i}}_{\bar{V}_{t,i}^{-1}},1\} \nonumber\\
    & \lep{d} \sum_{i \in [M]} \sum_{t \in \text{good epochs}} 2\sqrt{2} \beta_{T} \min\{ \norm{x_{t,i}}_{\bar{V}_{t,i}^{-1}},1\}\nonumber\\
     &\le \sum_{i \in [M]} \sum_{t\in [T] } 2\sqrt{2} \beta_{T} \min\{ \norm{x_{t,i}}_{\bar{V}_{t,i}^{-1}},1\}\nonumber\\
     &\lep{e} O\left(\beta_{T} \sqrt{dMT \log\left(1+\frac{MT}{d\lambda_{\min}}\right)} \right),
\end{align}
where (a) holds by the per-step regret bound in Step 2 and the boundedness of reward; (b) follows from the fact that $V_{t,i} \succeq G_{t,i} + \lambda_{\text{min}} I$  under event $\cE$;  (c) holds by~\eqref{eq:det-trick} when $t$ is in good epochs; (d) is true since $\beta_{T} \ge 1$; (e) holds by the elliptical potential lemma (cf. Lemma 11 in~\cite{abbasi2011improved}).

\textbf{Step 5: Bound the regret in bad epochs.} Let $T_{\text{bad}}$ be the total number of rounds in all bad epochs. Thus, the total number of bad rounds across \emph{all} agents are $M \cdot T_{\text{bad}}$. As a result, the cumulative group regret in all these bad rounds are upper bounded by $M \cdot T_{\text{bad}}$ due to the to the boundedness of reward. 

We are left to bound $T_{\text{bad}}$. All we need is to bound the $N_{\text{bad}}$ -- total number of bad epochs. Then, we have $T_{\text{bad}} = N_{\text{bad}} \cdot B$,  where $B$ is the fixed batch size. To this end, recall that $K = T/B$ and define ${\Psi} := \{k \in [K]: \log\det(V_k) - \log\det(V_{k-1}) > \log 2\}$, i.e., $N_{\text{bad}} = |{\Psi}|$. Thus, we have 
\begin{align*}
    \log 2 \cdot |{\Psi}| \le \sum_{k \in {\Psi}}  \log\det(V_k) - \log\det(V_{k-1}) \le \sum_{k \in [K]}  \log\det(V_k) - \log\det(V_{k-1}) \le d\log\left(1+\frac{MT}{d\lambda_{\min}}\right)
\end{align*}
Hence, we have $N_{\text{bad}} = |{\Psi}| \le \frac{d}{\log 2} \log\left(1+\frac{MT}{d\lambda_{\min}}\right)$. Thus we can bound the regret in bad epochs as follows.
\begin{align}
\label{eq:bad-regret}
    \sum_{i\in [M]} \sum_{t \in \text{bad epochs} } r_{t,i} &\le M\cdot T_{\text{bad}} =  M\cdot B\cdot N_{\text{bad}} \le M \cdot B \cdot \frac{d}{\log 2} \log\left(1+\frac{MT}{d\lambda_{\min}}\right).
\end{align}
\textbf{Step 6: Putting everything together.} Now, we substitute the total regret in good epochs given by~\eqref{eq:good-regret} and total regret in bad epochs given by~\eqref{eq:bad-regret} into the total regret decomposition in Step 3, yields the final cumulative group regret 
\begin{align*}
    \text{Reg}_M(T) = O\left(\beta_{T} \sqrt{dMT \log\left(1+\frac{MT}{d\lambda_{\min}}\right)} \right) + O\left(  M \cdot B \cdot {d}\log\left(1+\frac{MT}{d\lambda_{\min}}\right)\right),
\end{align*}
where 
$\beta_{T}:=\sqrt{2\log\left(\frac{1}{\alpha}\right) + d\log\left(1+\frac{MT}{d\lambda_{\min}}\right)} + \sqrt{\lambda_{\max}} + \nu$.
Finally, taking a union bound, we have the required result. 
\end{proof}
Now, we turn to the proof of Lemma~\ref{lem:subG}, which is an application of Lemma~\ref{lem:general} we just proved. 
\begin{proof}[Proof of Lemma~\ref{lem:subG}]
To prove the result, thanks to Lemma~\ref{lem:general}, we only need to determine the three constants $\lambda_{\max}, \lambda_{\min}$ and $\nu$ under the sub-Gaussian private noise assumption in Assumption~\ref{ass:subG}. To this end, we resort to concentration bounds for sub-Gaussian random vector and random matrix. 

To start with, under (i) in Assumption~\ref{ass:subG}, by the concentration bound for the norm of a vector containing sub-Gaussian entries (cf. Theorem 3.1.1 in~\cite{vershynin2018high}) and a union bound over all communication rounds, we have for all  $t = kB$ where $k = [T/B]$ and any $\alpha \in (0,1]$, with probability at least $1-\alpha/2$, for some absolute constant $c_1$,
\begin{align*}
    \norm{\sum_{i=1}^M n_{t,i}} = \norm{h_t} \le \Sigma_n:= c_1\cdot \widetilde{\sigma}_1 \cdot(\sqrt{d} + \sqrt{\log(T/(\alpha B)}).
\end{align*}

By (ii) in Assumption~\ref{ass:subG}, the concentration bound for the norm of a sub-Gaussian symmetric random matrix (cf. Corollary 4.4.8 in~\cite{vershynin2018high}) and a union bound over all communication rounds, we have for all $t = kB$ where $k = [T/B]$ and any $\alpha \in (0,1]$, with probability at least $1-\alpha/2$, 
\begin{align*}
    \norm{\sum_{i=1}^M N_{t,i}} \le \Sigma_N:= c_2\cdot \widetilde{\sigma}_2 \cdot(\sqrt{d} + \sqrt{\log(T/(\alpha B )})
\end{align*}
for some absolute constant $c_2$. Thus, if we choose $\lambda = 2 \Sigma_N$, we have $\norm{H_t} = \norm{\lambda I_d + \sum_{i=1}^M N_{t,i} } \le 3\Sigma_N$, i.e., $\lambda_{\max} = 3\Sigma_N$, and $\lambda_{\min} = \Sigma_N$. Finally, to determine $\nu$, we note that 
\begin{align*}
    \norm{h_t}_{H_t^{-1}} \le \frac{1}{\sqrt{\lambda_{\min}}} \norm{h_t} \le c\cdot \left(\sigma \cdot (\sqrt{d} + \sqrt{\log(T/(\alpha B)})\right)^{1/2} := \nu,
\end{align*}
where $\sigma= \max\{\widetilde{\sigma}_1,\widetilde{\sigma}_2\}$. The final regret bound is obtained by plugging the three values into the result given by Lemma~\ref{lem:general}.
\end{proof}

\section{Discussion on Private Adaptive Communication}
\label{app:priv-com}

In the main paper and Appendix~\ref{app:gap}, we have pointed out that the gap in privacy guarantee of Algorithm 1 in~\cite{dubey2020differentially} is that its adaptive communication schedule leads to privacy leakage due to its dependence on non-private data. As mentioned in Remark~\ref{rem:weak-FedDP}, one possible approach is to use private data to determine the sync in~\eqref{eq:sync-app}. This will resolve the privacy issue. However, the same issue in communication cost still remains (due to privacy noise), and hence $O(\log T)$ communication does not hold. Moreover, this new approach will also lead to new challenges in regret analysis, when compared with its current one in~\cite{dubey2020differentially} and the standard one in~\cite{wang2019distributed}.


To better illustrate the new challenges, let us restate Algorithm 1 in~\cite{dubey2020differentially} using our notations and first focus on how to establish the regret based on its current adaptive schedule (which has the issue of privacy leakage). After we have a better understanding of the idea, we will see how new challenges come up when one uses private data for an adaptive schedule.

\begin{algorithm}[!t]
   \caption{Restatement of Algorithm 1 in~\cite{dubey2020differentially}}
  \label{alg:FedLUCB-dubey}
\begin{algorithmic}[1]
  \STATE {\bfseries Parameters:} Adaptive communication parameter $D$, regularization $\lambda >0$, confidence radii $\{\beta_{t,i}\}_{t\in [T], i\in [M]}$, feature map $\phi_i:\cC_i \times \cK_i \to \Real^d$, privacy budgets $\epsilon>0, \delta \in [0,1]$.
  \STATE {\bfseries Initialize:} For all $i \in [M]$, $W_{i} = 0, U_{i} = 0$, {\priv} with $\epsilon, \delta$, $\widetilde{W}_{\text{syn}} = 0$, $\widetilde{U}_{\text{syn}} = 0$, 
 \FOR{$t\!=\!1, \ldots, T$}
  
    \FOR{each agent $i = 1,\ldots, M$}
      \STATE $V_{t,i} =  \lambda I+ \widetilde{W}_{\text{syn}} + W_{i}$, $\hat{\theta}_{t,i} = V_{t,i}^{-1}(\widetilde{U}_{\text{syn}} + U_{i}) $
      \STATE Play arm $a_{t,i} \!=\! \argmax_{a \in \cK_i} \inner{\phi_i(c_{t,i}, a)} {\hat{\theta}_{t,i}} + \beta_{t,i} \norm{\phi_i(c_{t,i},a)}_{V_{t,i}^{-1}}$ and set $x_{t,i} = \phi_i(c_{t,i}, a_{t,i})$
      \STATE Observe reward $y_{t,i}$
      \STATE Update $W_{i} = W_{i} + x_{t,i}x_{t,i}^{\top}$, $U_{i} = U_{i} + x_{t,i}y_{t,i}$
    \textcolor{DarkRed}{
      \IF{$\log\det({V}_{t,i} + x_{t,i}x_{t,i}^{\top}) - \log\det( V_{\text{last}}) > \frac{D}{ t - t_{\text{last}}}$} \alglinelabel{line:det-sync}
        \STATE Send a signal to the server to start a synchronization round.
      \ENDIF}   
      \IF{{a synchronization is started}}
        \STATE Send $W_{i}$ and $U_{i}$ to {\priv}
        \STATE {\priv} sends private cumulative statistics $\widetilde{W}_{t,i}, \widetilde{U}_{t,i}$ to server
        \STATE Server aggregates $\widetilde{W}_{\text{syn}} = \widetilde{W}_{\text{syn}} + \sum_{j=1}^M \widetilde{W}_{t,j}$ and $\widetilde{U}_{\text{syn}} = \widetilde{U}_{\text{syn}} + \sum_{j=1}^M \widetilde{U}_{t,j}$
        \STATE Receive $\widetilde{W}_{\text{syn}}$ and $\widetilde{U}_{\text{syn}}$ from the server
        \STATE Reset $W_{i} = 0$, $U_{i} = 0$, \textcolor{DarkRed}{$t_{\text{last}} = t$ and $V_{\text{last}} = \lambda I + \widetilde{W}_{syn}$} 
      \ENDIF
      
    \ENDFOR
  \ENDFOR
\end{algorithmic}
\end{algorithm}

As shown in Algorithm~\ref{alg:FedLUCB-dubey}, the key difference compared with our fixed-batch schedule is highlighted in color. Note that we only focus on silo-level LDP and use {\priv} to represent a general protocol that can privatize the communicated data (e.g., $\cP$ or the standard tree-based algorithm in~\cite{dubey2020differentially}).

\subsection{Regret Analysis under Non-private Adaptive Schedule}

In this section, we demonstrate the key step in establishing the regret with the non-private adaptive communication schedule in Algorithm~\ref{alg:FedLUCB-dubey} (i.e., line 9). It turns out that the regret analysis is very similar to our proof for Lemma~\ref{lem:general} for the fixed batch case, in that the only key difference lies in Step 5 when bounding the regret in bad epochs\footnote{There is another subtle but important difference, which lies in the construction of filtration that is required to apply the standard self-normalized inequality to establish the concentration result. We believe that one cannot directly use the standard filtration (e.g.,~\cite{abbasi2011improved}) in the adaptive case, and hence more care is indeed required.}. The main idea behind adaptive communication is: \emph{whenever the accumulated local regret at any agent exceeds a threshold, then synchronization is required to keep the data homogeneous among agents.} This idea is directly reflected in the following analysis.



\textbf{Bound the regret in bad epochs (adaptive communication case).} 
Let's consider an arbitrary bad epoch $k$, i.e., $(t_{k-1}, t_k]$, where $t_k$ is the round for the $k$-th communication. For all $i$, we want to bound the total regret between $(t_{k-1}, t_k]$, denoted by $R_i^k$. That is, the local regret between any two communications (in the bad epoch) will not be too large. For now, suppose we already have such a bound $U$ (which will be achieved by adaptive communication later), i.e., $R_i^k \le U$ for all $i, k$, we can easily bound the total regret in bad epochs. To see this, recall that ${\Psi} := \{k \in [K]: \log\det(V_k) - \log\det(V_{k-1}) > \log 2\}$, i.e., $N_{bad} = |{\Psi}|$, we have 
\begin{align*}
    \sum_i \sum_{t \in \text{bad epochs} } r_{t,i} = \sum_i \sum_{k \in \Psi } R_i^k = O\left( |\Psi| M  U\right).
\end{align*}
Plugging in $N_{bad} = |{\Psi}| \le \frac{d}{\log 2} \log\left(1+\frac{MT}{d\lambda}\right)$, we have the total regret for bad epochs. Now, we are only left to find $U$. Here is the place where the adaptive schedule in the algorithm comes in. 
First, note that
\begin{align}
   \sum_{t_{k-1} < t < t_k} r_{t,i} &\lep{a}  \sum_{t_{k-1} < t < t_k}  \min\{ 2\beta_{T}\norm{x_{t,i}}_{V_{t,i}^{-1}},1\}\label{eq:5} \\
   &\lep{b} O\left(\beta_{T}\sqrt{(t_k - t_{k-1}) \log \frac{\det V_{t_k,i}}{\det V_{{\text{last}}}}}\right)\label{eq:det}\\
   &\lep{c} O\left(\beta_T \sqrt{D}\right)\nonumber,
\end{align}
where (a) holds by boundedness of reward; (b) follows from the elliptical potential lemma, i.e., $V_{\text{last}}$ is PSD under event $\cE$ and $V_{t,i} = V_{t-1,i} + x_{t-1,i}x_{t-1,i}^{\top}$ for all $t \in (t_{k-1}, t_k)$; 
(c) holds by the adaptive schedule in line 9 of Algorithm~\ref{alg:FedLUCB-dubey}.
As a result, we have $R_i^k \le  O\left(\beta_{T} \sqrt{D}\right) + 1$, where the regret at round $t_k$ is at most $1$ by the boundedness of reward. With a proper choice of $D$, one can obtain a final regret bound. 

\subsection{Challenges in Regret Analysis under Private Adaptive Schedule}
Now, we will discuss new challenges when one uses private data for an adaptive communication schedule. In this case, one needs to first privatize the new local gram matrices (e.g., $\sum_{s=t_{\text{last}}+1}^t x_{s,i}x_{s,i}^{\top}$) before being used in the determinant condition. This can be done by using standard tree-based algorithm with each data point as $x_{s,i}x_{s,i}^{\top}$. With this additional step, now the determinant condition becomes 
\begin{align}
\label{eq:det-new}
    \log\det(\widetilde{V}_{t,i}) - \log\det( V_{\text{last}}) > \frac{D}{ t - t_{\text{last}}},
\end{align}
where $\widetilde{V}_{t,i} := V_{\text{last}} + \sum_{s = t_{\text{last}}+1}^ t x_{s,i}x_{s,i}^{\top} + N^{\text{loc}}_{t,i}$ and $N_{t,i}^{\text{loc}}$ is the new local injected noise for private schedule up to time $t$. Now suppose one uses~\eqref{eq:det-new} to determine $t_{k}$. Then, it does not imply that~\eqref{eq:det} is upper bounded by $\beta_T\sqrt{D}$. That is, $\frac{\det(\widetilde{V}_{t,i})}{\det( V_{\text{last}})} \le D'$ does not necessarily mean that $\frac{\det(V_{\text{last}} + \sum_{s = t_{\text{last}}+1}^ t x_{s,i}x_{s,i}^{\top})}{\det( V_{\text{last}})} \le D'$.

One may try to work around~\eqref{eq:5} by first using $G_{t,i} + \lambda_{\min} I$ to lower bound $V_{t,i}$. Then,~\eqref{eq:det} becomes $O\left(\beta_{T}\sqrt{(t_k - t_{k-1}) \log \frac{\det (G_{t_k,i} + \lambda_{\min}I)}{\det (G_{t_{k-1},i}+\lambda_{\min} I)}}\right) $, which again cannot be bouded based on the rule given by~\eqref{eq:det-new}. To see this, note that $\frac{\det(\widetilde{V}_{t_k-1,i})}{\det( V_{\text{last}})} \le D'$  only implies that $\frac{\det (G_{t_k,i} + \lambda_{\min}I)}{\det (G_{t_{k-1},i}+\lambda_{\max} I)} \le D'$.

\section{Additional Details on Federated LCBs under Silo-Level LDP}
\label{app:LDP}

In this section, we provide details for Section~\ref{sec:main-LDP}. In particular, we present the proof for Theorem~\ref{thm:LDP} and the alternative privacy protocol for silo-level LDP. 

\subsection{Proof of Theorem~\ref{thm:LDP}}
\begin{proof}[Proof of Theorem~\ref{thm:LDP}]

\textbf{Privacy.} We only need to show that $\cP$ in Algorithm~\ref{alg: p-sdp} with a proper choice of $\sigma_0$ satisfies $(\epsilon,\delta)$-DP for all $k \in [K]$, which implies that the full transcript of the communication is private in Algorithm~\ref{alg:FedLUCB-SDP} for any local agent $i$. 


First, we recall that the (multi-variate) Gaussian mechanism satisfies zero-concentrated differential privacy (zCDP)~\cite{bun2016concentrated}. In particular, by~\cite[Lemma 2.5]{bun2016concentrated}, we have that computation of each node (p-sum) in the tree is $\rho$-zCDP with $\rho = \frac{L^2}{2\sigma_0^2}$.
Then, from the construction of the binary tree in $\cP$, one can easily see that one single data point $\gamma_i$ (for all $i \in [K]$) only impacts at most $1+\log(K)$ nodes. Thus, by \emph{adaptive} composition of zCDP (cf. Lemma 2.3 in~\cite{bun2016concentrated}), we have that  the entire releasing of all p-sums is $(1+\log K) \rho$-zCDP. Finally, we will use the conversion lemma from zCDP to approximated DP (cf. Proposition 1.3 in~\cite{bun2016concentrated}). In particular, we have that $\rho_0$-zCDP implies $(\epsilon = \rho_0 + 2\sqrt{\rho_0\cdot \log(1/\delta)}, \delta)$-DP for all $\delta > 0$. In other words, to achieve a given $(\epsilon,\delta)$-DP, it suffices to achieve $\rho_0$-zCDP with $\rho_0 = f(\epsilon,\delta):=(\sqrt{\log(1/\delta) + \epsilon} - \sqrt{\log(1/\delta)})^2$. In our case, we have $\rho_0 = (1+\log (K)) \rho = (1+\log(K))\frac{L^2}{2\sigma_0^2}$. Thus, we have $\sigma_0^2 =  (1+\log(K))\frac{L^2}{2\rho_0} = (1+\log(K))\frac{L^2}{2f(\epsilon,\delta)}$. To simply it, one can lower bound $f(\epsilon,\delta)$ by $\frac{\epsilon^2}{4\log(1/\delta) + 4\epsilon}$ (cf. Remark 15 in~\cite{steinke2022composition}). Therefore, to obtain $(\epsilon,\delta)$-DP, it suffices to set $\sigma_0^2 = 2\cdot L^2 \cdot \frac{(1+\log(K))(\log(1/\delta) + \epsilon)}{\epsilon^2}$.
Note that there are two streams of data in Algorithm~\ref{alg:FedLUCB-SDP}, and hence it suffices to ensure that each of them is $(\epsilon/2,\delta/2)$-DP. This gives us the final noise level $\sigma_0^2 = 8 \frac{(1+\log(K))(\log(2/\delta) + \epsilon)}{\epsilon^2}$ (note that by boundedness assumption $L = 1$ in our case).


\textbf{Regret.} In order to establish the regret bound, thanks to Lemma~\ref{lem:subG}, we only need to determine the maximum noise level in the learning process. Recall that ${\sigma}_0^2 = 8\cdot\frac{(1+\log(K))(\log(2/\delta) + \epsilon)}{\epsilon^2}$ is the noise level for both streams (i.e., $\gamma^{\text{bias}}$ and $\gamma^{\text{cov}}$). Now, by the construction of  binary tree in $\cP$, one can see that each prefix sum $\sum[1,k]$ only involves at most $1+\log(k)$ tree nodes. Thus, we have that the noise level in $n_{t,i}$ and $N_{t,i}$ are upper bounded by $(1+\log(K)) {\sigma}_0^2$. As a result, the overall noise level across all $M$ silos is upper bounded by $\sigma_{\text{total}}^2 = M (1+\log (K)) {\sigma}_0^2$. Finally, setting $\sigma^2$ in Lemma~\ref{lem:subG} to be the noise level $\sigma_{\text{total}}^2$ , yields the required result.  
\end{proof}

\subsection{Alternative Privacy Protocol for Silo-Level LDP}
\label{app:alt}
For silo-level LDP, each local randomizer can simply be the standard tree-based algorithm, i.e., releasing the prefix sum at each communication step $k$ (rather than p-sum in Algorithm~\ref{alg: p-sdp}). The analyzer now becomes a simple aggregation. As before, no shuffler is required in this case. This alternative protocol is given by Algorithm~\ref{alg: p-ldp-alt}, which is essentially the main protocol used in~\cite{dubey2020differentially}.

It can be seen that both privacy and regret guarantees under this $\cP_{\text{alt}}$ are the same as Theorem~\ref{thm:LDP}. To see this, for privacy, the prefix sum is a post-processing of the p-sums. Thus, since we have already shown that the entire releasing of p-sums is private in the proof of Theorem~\ref{thm:LDP}, hence the same as the prefix sum. Meanwhile, the total noise level at the server is the same as before. Thus, by~Lemma~\ref{lem:subG}, we have the same regret bound.

\begin{algorithm}[!t]
\caption{$\mathcal{P}_{\text{alt}}$, an alternative privacy protocol for silo-level LDP}
\label{alg: p-ldp-alt}
\begin{algorithmic}[1]
\STATE {\bfseries Procedure:} Local Randomizer $\mathcal{R}$
\STATE{\textcolor{DarkBlue}{\texttt{// Input: stream data $\gamma = (\gamma_i)_{i\in [K]}$; privacy parameters $\epsilon, \delta$; Output: private prefix sum}}}
\begin{ALC@g}
    \FOR{$k\!=\!1, \ldots, K$}
       \STATE Express $k$ in binary form: $k = \sum_j \text{Bin}_j(k) \cdot 2^j$
       \STATE Find the index of first one $i_k: = \min\{j: \text{Bin}_j(k) =1\}$
       \STATE Compute p-sum $\alpha_{i_k} = \sum_{j < i} \alpha_j + \gamma_k$.
       \STATE Add noise to p-sum $\hat{\alpha}_{i_k} = \alpha_{i_k} + \cN(0,\sigma_0^2 I)$ 
        \STATE Output private prefix sum $\widetilde{s}_k = \sum_{j: \text{Bin}_j(k)=1} \hat{\alpha}_j$
  \ENDFOR
\end{ALC@g}
\STATE {\bfseries end procedure}
\STATE {\bfseries Procedure: Analyzer $\mathcal{A}$ } 
\STATE{\textcolor{DarkBlue}{\texttt{// Input: a collection of $M$ data points, $y = \{y_i\}_{i\in [M]}$; Output: Aggregated sum}}}
\begin{ALC@g}
  \STATE Output $\widetilde{y} = \sum_{i\in [M]} y_i$
\end{ALC@g}
\STATE {\bfseries end procedure}
\end{algorithmic}
\end{algorithm}

\section{Additional Details on Federated LCBs under SDP}
\label{app:SDP}
In this section, we provide more detailed discussions on SDP and present the proof for Theorem~\ref{thm:SDP} (SDP via amplification lemma) and Theorem~\ref{thm:SDP-vec} (SDP via vector sum). 

First, let us start with some general discussions. 

\textbf{Importance of communicating P-sums.}
For SDP, it is important to communicate P-sums rather than prefix sum. 
Note that communicating noisy p-sums in our privacy protocol $\cP$ rather than the noisy prefix sum (i.e., the sum from beginning as done in~\cite{dubey2020differentially}) plays a key role in achieving optimal regret with shuffling. To see this, both approaches can guarantee silo-level LDP. By our new amplification lemma, privacy guarantee can be amplified by $1/\sqrt{M}$ in $\epsilon$ for each of the $K$ shuffled outputs, where $K=T/B$ is total communication rounds. Now, if the prefix sum is released to the shuffler, then any single data point participates in at most $K$ shuffle mechanisms, which would blow up $\epsilon$ by a factor of $O(\sqrt{K})$ (by advanced composition~\cite{dwork2014algorithmic}). This would eventually lead to a $K^{1/4}$ factor blow up in regret due to privacy. Similarly, if we apply $\cP_{\text{Vec}}$ to the data points in the prefix sum, then again a single data point can participate in at most $K$ shuffled outputs. 

On the other hand, if only noisy p-sums are released for shuffling at each communication round $k\in [K]$ (as in our protocol $\cP$) or only the data points in each p-sum are used in $\cP_{\text{Vec}}$ (as in our protocol in $\cP^{\cT}_{\text{Vec}}$), then due to the binary-tree structure, each data point only participates in at most $\log K$ shuffled mechanisms, which only leads to $O(\sqrt{\log K})$ blow-up of $\epsilon$; hence allowing us to achieve the desired $\widetilde{O}(\sqrt{MT})$ regret scaling, and 
close the gap present under silo-level LDP. 

\begin{remark}[Shuffled tree-based mechanism]
Both the protocol $\cP$ in Algorithm~\ref{alg: p-sdp} along with our new amplification lemma and protocol $\cP_{\text{Vec}}^{\cT}$ in Algorithm~\ref{alg: p-vec-T} can be treated as a black-box method, which integrates shuffling into the tree-based mechanism while providing formal guarantees for continual release of sum statistics. Hence, it can be applied to other federated online learning problems beyond contextual bandits.
\end{remark}

\subsection{Amplification lemma for SDP}
We first formally introduce our new amplification lemma, which is the key to our analysis, as mentioned in the main paper.

The motivation for our new amplification result is two-fold: (i) Existing results on privacy amplification via shuffling (e.g.,~\cite{feldman2022hiding,erlingsson2019amplification,cheu2019distributed,balle2019privacy}) are only limited to the standard LDP case, i.e., each local dataset has size $n=1$, which is not applicable in our case where each silo runs a DP (rather than LDP) mechanism over a dataset of size $n=T$;
(ii) Although a recent work~\cite{lowy2021private} establishes a general amplification result for the case of $n>1$, it introduces a very large value for the final $\delta$ that scales linearly with $n$ due to group privacy. 

We first present the key intuition behind our new lemma. Essentially, as in~\cite{lowy2021private}, we follow the nice idea of hiding among the clones introduced in~\cite{feldman2022hiding}. That is, the output from silo $2$ to $n$ can be similar to that of silo $1$ by the property of DP (i.e., creating clones). The key difference between $n=1$ and $n>1$ is that in the latter case, the similarity distance between the output of silo $1$ and $j$ ($j>1$) will be larger as in this case all $n>1$ data points among two silos could be different. To capture this,~\cite{lowy2021private} resorts to group privacy for general DP local randomizers.\footnote{This is because it mainly focuses on the lower bound, where one needs to be general to handle any mechanisms.} However, group privacy for approximate DP will introduce a large value for $\delta$. Thus, since we know that each local randomizer in our case is the Gaussian mechanism, we can capture the similarity of outputs between silo $1$ and $j$ ($j>1$) by directly bounding the sensitivity. This helps to avoid the large value for the final $\delta$. Specifically, we have the following result, which can be viewed as a refinement of Theorem D.5 in~\cite{lowy2021private} when specified to the Gaussian mechanism. We follow the notations in~\cite{lowy2021private} for easy comparison. 

\begin{lemma}[Amplification lemma for Gaussian mechanism]
\label{lem:amp}
Let $\bx = (X_1, \cdots, X_N) \in \XX^{N \times n}$ be a distributed data set, i.e., $N$ silos each with $n$ data points. Let  $r \in \mathbb{N}$ and let $\rand_r(\bz, \cdot): \mathcal{X}^n \to \mathcal{Z} := \Real^d$ be a Gaussian mechanism with $(\epsor, \delor)$-DP, $\epsor \in (0,1)$\footnote{Note that standard Gaussian mechanism only applies to the regime when $\epsilon <1$. In our case, $\epsilon_0^r$ is often less than $1$. Gaussian mechanism also works for the regime $\epsilon > 1$, in this case, $\sigma^2 \approx 1/\epsilon$ rather than $1/\epsilon^2$. With minor adjustment of the final $\epsilon^r$, our proof can be extended. },  for all $\bz = Z_{(1:r-1)}^{(1:N)} \in \mathcal{Z}^{(r-1) \times N}$ and $i \in [N],$ where $\XX$ is an arbitrary set.
Suppose for all $i$, $\max_{\text{any pair}(X,X')}\norm{\rand_r(\bz, X) - \rand_r(\bz, X')} \le n \cdot \max_{\text{adjacent pair}(X,X')}\norm{\rand_r(\bz, X) - \rand_r(\bz, X')}$.\footnote{This is w.l.o.g; one can easily generalize it to any upper bound that is a function of $n$.} Given $\bz = Z_{(1:r-1)}^{(1:N)}$, consider the shuffled algorithm $\mathcal{A}^r_s: \XX^{n \times N} \times \ZZ^{(r-1) \times N} \to \mathcal{Z}^N$ that first samples a random permutation $\pi$ of $[N]$ and then computes $Z_r = (Z^{(1)}_r, \cdots, Z^{(N)}_r),$ where 
$Z^{(i)}_r := \rand_{r}(\bz, X_{\pi(i)}).$ Then, for any $\delta \in [0,1]$ such that $\epsilon^r_0 \le \frac{1}{n}\ln\left(\frac{N}{16 \log(2/\delta)}\right)$,  $\mathcal{A}^r_s$ is $(\epsilon^r, {\delta}^r)$-DP, where \begin{align*} 
\label{eq: epsilon r amplification}
    \epsilon^r &:= \ln\left[1 + \left(\frac{e^{\epsor} - 1}{e^{\epsor} + 1}\right)\left(\frac{8 \sqrt{e^{n\epsor}\log(4/\delta)}}{\sqrt{N}} + \frac{8 e^{n\epsor}}{N} \right) 
    \right] \\ \delta^r &:=\left(\frac{e^{\epsor} - 1}{e^{\epsor} + 1}\right)\delta + N (e^{\epsilon^r}+1)(1+e^{-\epsilon_0^r}/2)\delor.
\end{align*}
If $\epsor \le 1/n$, choosing $\delta = N n \delta_0^r$ yields
$\epsilon^r = O\left(\frac{\epsilon_0^r \sqrt{\log(1/(n N \delta_0^r))}}{\sqrt{N}}\right)$
and $\delta^r = O(N\delta_0^r)$, where $\delta_0^r \le 1/(N n)$.
\end{lemma}

\subsection{Vector Sum Protocol for SDP}
\label{app:pvec}

One limitation of our first scheme for SDP is that the privacy guarantee holds only for very small values of $\epsilon$. 
This comes from two factors: one is due to the fact that standard $1/\sqrt{M}$ amplification result requires the local privacy budget to be close to one; the other one comes from the fact that now the local dataset could be $n = T$, which further reduces the range of valid $\epsilon$.

\begin{algorithm}[!t]
\caption{$\mathcal{P}_{\text{Vec}}^{{\cT}
}$, another privacy protocol used in Algorithm~\ref{alg:FedLUCB-SDP}}
\label{alg: p-vec-T}
\begin{algorithmic}[1]
%
\STATE {\bfseries Procedure:} Local Randomizer $\mathcal{R}$ at each agent
\STATE{\textcolor{DarkBlue}{\texttt{// Input: stream data $ (\gamma_1,\ldots,\gamma_K)$, privacy budgets $\epsilon > 0, \delta \in (0,1]$}}}
\begin{ALC@g}
    \FOR{$k\!=\!1, \ldots, K$}
       \STATE Express $k$ in binary form: $k = \sum_j \text{Bin}_j(k) \cdot 2^j$
       \STATE Find index of first one $i_k \!=\! \min\{j: \text{Bin}_j(k) \!=\! 1\}$
        \STATE Let $\cD_k$ be the set of all data points\footnotemark that contribute to $\alpha_{i_k} = \sum_{j < i_k} \alpha_j + \gamma_k$
       \STATE Output $y_k = \cR_{\text{Vec}}(\cD_k) $  {\textcolor{DarkBlue}{\texttt{// apply $R_{\text{Vec}}$ in Algorithm~\ref{alg: Pvec} to each data point}}}
  \ENDFOR
\end{ALC@g}
\STATE {\bfseries Procedure:} Analyzer $\mathcal{A}$ at server
\STATE{\textcolor{DarkBlue}{\texttt{// Input: stream data from $\cS$: $\lbrace\bar{y}_k=(\bar{y}_{k,1},\ldots,\bar{y}_{k,M})\rbrace_{k \in [K]}$ }}}
\begin{ALC@g}
   \FOR{$k\!=\!1, \ldots, K$}
       \STATE Express $k$ in binary and find index of first one $i_k$
        \STATE Add all messages from $M$ agents: $\widetilde{\alpha}_{i_k} = \cA_{\text{Vec}}(\bar{y}_k)$ {\textcolor{DarkBlue}{\texttt{// apply $A_{\text{Vec}}$ in Algorithm~\ref{alg: Pvec}}}}
        \STATE Output: $\widetilde{s}_k = \sum_{j: \text{Bin}_j(k)=1} \widetilde{\alpha}_j$
  \ENDFOR
\end{ALC@g}
\end{algorithmic}
\end{algorithm}
\footnotetext{In our application, each data point means a bias vector or a covariance matrix. See Appendix~\ref{app:pvec} for a concrete example.}

In this section, we give the vector sum protocol in~\cite{cheu2021shuffle} for easy reference. Let's also give a concrete example to illustrate how to combine Algorithm~\ref{alg: Pvec} with Algorithm~\ref{alg: p-vec-T}. Consider a fixed $k = 6$. Then, for each agent, we have $\alpha_{i_6} = \gamma_5 + \gamma_6$. That is, consider the case of summing bias vectors, for agent $i \in [M]$, $\gamma_5 = \sum_{t= 4B+1}^{5B} x_{t,i}y_{t,i}$ and $\gamma_6 = \sum_{t= 5B+1}^{6B} x_{t,i}y_{t,i}$. Then, $\cD_6$ consists of $2B$ data points, each of which is a single bias vector. Now, $\cR_{\text{vec}}$ and $\cA_{\text{vec}}$ (as well the shuffler) work together to compute the noisy sum of $2B\cdot M$ data points. In particular, denote by $\cP_{\text{vec}}$ the whole process, then we have $\widetilde{\alpha}_{i_6} = \cP_{\text{vec}}(\cD_6^M)$, where $\cD_6^M$ is the data set that consists of $n= 2B\cdot M$ data points, each of them is a single bias vector.

Next, we present more details on the implementations, i.e., the parameter choices of $g, b, p$. Let's consider $k=6$ again as an example. In this case, the total number of data points that participate in $\cP_{\text{vec}}$ is $n = 2B\cdot M$. Then, according to the proof of Theorem C.1 in~\cite{pmlr-v162-chowdhury22a}, we have 
\begin{align*}
    g = \max\{2\sqrt{n}, d, 4\}, \quad b = \frac{24\cdot 10^4\cdot g^2\cdot \left(\log\left(\frac{4\cdot(d^2+1)}{\delta}\right)\right)^2}{\epsilon^2n},\quad p = 1/4.
\end{align*}

\begin{algorithm}
\caption{${P}_{\text{vec}}$, a shuffle protocol for vector summation \cite{cheu2021shuffle}}
\label{alg: Pvec}
\begin{algorithmic}[1]
\STATE {\bfseries Input:} Database of $d$-dimensional vectors $\mathbf{X}= (\mathbf{x}_1, \cdots, \mathbf{x}_n$); privacy parameters $\epsilon, \delta$; $L$.
\STATE {\bfseries procedure:} Local Randomizer ${R}_{\text{vec}}(\mathbf{x}_i)$
\begin{ALC@g}
\FOR{$j \in [d]$} 
\STATE Shift component to enforce non-negativity: $\mathbf{w}_{i,j} \gets \mathbf{x}_{i,j} + L$
\STATE $\mathbf{m}_j \gets \mathcal{R}_{1D}(\mathbf{w}_{i,j})$
\ENDFOR
\STATE Output labeled messages $\{(j, \mathbf{m}_j)\}_{j \in [d]}$
\end{ALC@g}
\STATE {\bfseries end procedure}
\STATE {\bfseries procedure: Analyzer} ${A}_{\text{vec}}(\mathbf{y})$ 
\begin{ALC@g}
\FOR{$j \in [d]$}
\STATE Run analyzer on coordinate $j$'s messages $z_j \gets \mathcal{A}_{\text{1D}}(\mathbf{y}_j)$ 
\STATE Re-center: $o_j \gets z_j - n \cdot L$
\ENDFOR
\STATE Output the vector of estimates $\mathbf{o} = (o_1, \cdots o_d)$
\end{ALC@g}
\STATE {\bfseries end procedure}
\end{algorithmic}
\end{algorithm}

\begin{algorithm}
\caption{$\mathcal{P}_{\text{1D}}$, a shuffle protocol for summing scalars \cite{cheu2021shuffle}}
\label{alg: P1D}
\begin{algorithmic}[1]
\STATE {\bfseries Input:} 
Scalar database $X = (x_1, \cdots x_n) \in [0,L]^n$; $g, b \in \mathbb{N}; p \in (0, \frac{1}{2})$. 
\STATE {\bfseries procedure: Local Randomizer $\mathcal{R}_{1D}(x_i)$}

\begin{ALC@g}
\STATE $\bar{x}_i \gets \lfloor x_i g/L \rfloor$.
\STATE Sample rounding value $\eta_1 \sim \textbf{Ber}(x_i g/L - \bar{x}_i)$.
\STATE Set $\hat{x}_i \gets \bar{x}_i + \eta_1$.
\STATE Sample privacy noise value $\eta_2 \sim \textbf{Bin}(b,p)$.
\STATE Report $\mathbf{y}_i \in \{0,1\}^{g + b}$ containing $\hat{x}_i + \eta_2$ copies of $1$ and $g + b - (\hat{x}_i + \eta_2)$ copies of $0$.
\end{ALC@g}
\STATE {\bfseries end procedure}
\STATE{\bfseries procedure: Analyzer} $\mathcal{A}_{\text{1D}}(\mathcal{S}(\mathbf{y}_1,\ldots,\mathbf{y}_n))$
\begin{ALC@g}
\STATE Output estimator $\frac{L}{g}((\sum_{i=1}^{n}\sum_{j=1}^{b+g} (\mathbf{y}_i)_j) - pbn)$.
\end{ALC@g}
\STATE {\bfseries end procedure}
\end{algorithmic}
\end{algorithm}

\subsection{Proofs}
First, we present proof of Theorem~\ref{thm:SDP}.

\begin{proof} [Proof of Theorem~\ref{thm:SDP}]
\textbf{Privacy.}
In this proof, we directly work on approximate DP. By the boundedness assumption and Gaussian mechanism, we have that with ${\sigma}_0^2 = \frac{2L^2 \log(1.25/\hat{\delta}_0)}{\hat{\epsilon}_0^2} $, $\cR$ in $\cP$ is $(\hat{\epsilon}_0, \hat{\delta}_0)$-DP for each communication round $k \in [K]$ (provided $\hat{\epsilon}_0 \le 1$) . Now, by our amplification lemma (Lemma~\ref{lem:amp}), we have that the shuffled output is $(\hat{\epsilon}, \hat{\delta})$-DP with $\hat{\epsilon} = O\left(\frac{\hat{\epsilon}_0 \sqrt{\log(1/(T M \hat{\delta}_0))}}{\sqrt{M}}\right)$ and $\hat{\delta} = O(M \hat{\delta}_0)$ (provided $\hat{\epsilon}_0 \le 1/T$ and $\hat{\delta}_0 \le 1/(M T)$). Here we note that in our case,  $N = M$ and $n= T$, where $n=T$ follows from the fact that there exists  $\alpha_i$ in the tree that corresponds to the sum of $T$ data points. Moreover, since the same mechanism is run at all silos, shuffling-then-privatizing is the same as first privatizing-then-shuffling the outputs.  Next, we apply the advanced composition theorem (cf. Theorem 3.20 in~\cite{dwork2014algorithmic}). In particular, by the binary tree structure, each data point involves only $\kappa:=1+\log(K)$ times in the output of $\cR$. Thus, to achieve $(\epsilon,\delta)$-DP, it suffices to have $\hat{\epsilon} = \frac{\epsilon}{2\sqrt{2\kappa \log(2/\delta)}}$ and $\hat{\delta} = \frac{\delta}{2\kappa}$. Using all these equations, we can solve for $\hat{\epsilon}_0 = C_1\cdot \frac{\epsilon \sqrt{M}}{\sqrt{\kappa\log(1/\delta)\log(\kappa/(\delta T))}}$ and $\hat{\delta}_0 = C_2\cdot \frac{\delta}{ M\kappa}$, for some constants $C_1>0$ and $C_2>0$.
To satisfy the conditions on $\hat{\epsilon}_0$ and $\hat{\delta}_0$, we have 
$\epsilon \le \frac{\sqrt{\kappa}}{C_1T\sqrt{M}}$
and $\delta \le \frac{\kappa}{C_2 T}$. With the choice of $\hat{\epsilon}_0$ and $\hat{\delta}_0$, we have the noise variance $\sigma_0^2 = O\left(\frac{2 L^2 \beta \log(1/\delta)\log(\kappa/(\delta T))\log(M\kappa/\delta)}{\epsilon^2 M}\right)$. Thus, we can apply $\cP$ to the bias and covariance terms  (with $L = 1$), respectively. 

\textbf{Regret.} Again, we simply resort to our Lemma~\ref{lem:subG} for the regret analysis. In particular, we only need to determine the maximum noise level in the learning process. Note that $\sigma_0^2 = O\left(\frac{2 L^2 \kappa \log(1/\delta)\log(\kappa/(\delta T))\log(M\kappa/\delta)}{\epsilon^2 M}\right)$ is the noise level injected for both bias and covariance terms. Now, by the construction of the binary tree in $\cP$, one can see that each prefix sum only involves at most $1+\log(k)$ tree nodes. 
As a result, the overall noise level across all $M$ silos is upper bounded by $\sigma_{\text{total}}^2 = M \kappa {\sigma}_0^2$. Finally, setting $\sigma^2$ in Lemma~\ref{lem:subG} to be the noise level $\sigma_{\text{total}}^2$ , yields the required result.  
\end{proof}
Now, we prove Theorem~\ref{thm:SDP-vec}.

\begin{proof}[Proof of Theorem~\ref{thm:SDP-vec}]
\textbf{Privacy.} For each calculation of the noisy synchronized p-sum, there exist parameters for $\cP_{\text{Vec}}$ such that it satisfies $(\epsilon_0,\delta_0)$-SDP where $\epsilon_0 \in (0,15]$ and $\delta_0 \in (0,1/2)$ (see Lemma 3.1 in~\cite{cheu2021shuffle} or Theorem 3.5 in~\cite{pmlr-v162-chowdhury22a}). Then, by the binary tree structure, each single data point (bias vector or covariance matrix) only participates in at most $\kappa:=1+\log(K)$ runs of $\cP_{\text{Vec}}$.  Thus, to achieve $(\epsilon,\delta)$-DP, it suffices to have ${\epsilon_0} = \frac{\epsilon}{2\sqrt{2\kappa \log(2/\delta)}}$ and ${\delta_0} = \frac{\delta}{2\kappa}$ by advanced composition theorem. Thus, for any $\epsilon \in (0, 30\sqrt{2\kappa \log(2/\delta)})$ and $\delta \in (0,1)$, there exist parameters for  $\cP_{\text{Vec}}$ such that the entire calculations of noisy p-sums are $(\epsilon,\delta)$-SDP. Since we have two streams of data (bias and covariance), we finally have that for any $\epsilon \in (0, 60\sqrt{2\kappa \log(2/\delta)})$ and $\delta \in (0,1)$, there exist parameters for  $\cP_{\text{Vec}}$ such that Algorithm~\ref{alg:FedLUCB-SDP} with $\cP_{\text{Vec}}^{\cT}$ satisfies $(\epsilon,\delta)$-SDP.

\textbf{Regret.} By the same analysis in the proof of Theorem 3.5 in~\cite{pmlr-v162-chowdhury22a}, the injected noise for each calculation of the noisy \emph{synchronized} p-sum is sub-Gaussian with the variance being at most $\hat{\sigma}^2 = O\left(\frac{\log^2(d^2/\delta_0)}{\epsilon_0^2}\right) = O\left(\frac{\kappa \log (1/\delta) \log^2(d^2\kappa/\delta)}{\epsilon^2}\right)$. Now, by the binary tree structure, each prefix sum only involves at most $\kappa$ p-sums. Hence, the overall noise level is upper bounded by $\sigma_{\text{total}}^2 = \kappa \hat{\sigma}^2$. Finally, setting $\sigma^2$ in Lemma~\ref{lem:subG} to be the noise level $\sigma_{\text{total}}^2$ , yields the required result.  
\end{proof}

Now, we provide proof of amplification Lemma~\ref{lem:amp} for completeness. We follow the same idea as in~\cite{feldman2022hiding} and~\cite{lowy2021private}. For easy comparison, we use the same notations as in~\cite{lowy2021private} and highlighted the key difference using color text.

\begin{proof}[Proof of Lemma~\ref{lem:amp}]
Let $\bx_0, \bx_1 \in \XX^{n \times N}$ be adjacent distributed data sets (i.e. $\sum_{i=1}^N \sum_{j=1}^n \mathbbm{1}_{\{ x_{i,j} \neq x_{i,j}\}} = 1$). Assume WLOG that $\bx_0 = (X_1^0, X_2, \cdots, X_N)$ and $\bx_1 = (X_1^1, X_2, \cdots, X_N),$ where $X_1^0 = (x_{1,0}, x_{1,2}, \cdots, x_{1,n}) \neq (x_{1,1}, x_{1,2}, \cdots, x_{1,n}).$ We can also assume WLOG that $X_j \notin \{X_1^0, X_1^1\}$ for all $j \in \{2, \cdots, N\}$ by re-defining $\XX$ and $\rand_r$ if necessary. 

Fix $i \in [N], r \in [R], \bz = \bz_{1:r-1} = Z_{(1:r-1)}^{(1:N)} \in \ZZ^{(r-1) \times N}$, denote $\mathcal{R}(X):= \rand_r(\bz, X)$ for $X \in \XX^n,$ and $\Al_s(\bx):= \Al_s^r( \bz_{1:r-1}, \bx).$  Draw $\pi$ uniformly from the set of permutations of $[N]$. Now, since $\RR$ is $(\epsilon_0^r, \delta_0^r)$-DP,
$\RR(\Xoo) \edorsim  \RR(\Xoz)$, so by~\cite[Lemma D.12]{lowy2021private}, there  exists a local randomizer $\RR'$ such that $\RR'(\Xoo) \edorzsim \RR(\Xoz)$ and $TV(\RR'(\Xoo), \RR(\Xoo)) \leq \delor.$

Hence, by~\cite[Lemma D.8]{lowy2021private}, there exist distributions $U(X_1^0)$ and  $U(X_1^1)$ such that \begin{equation}
\label{eq: 16}
    \RR(\Xoz) = \frac{e^{\epsor}}{e^{\epsor} + 1}U(\Xoz) + \frac{1}{e^{\epsor} + 1}U(\Xoo) 
\end{equation}
and \begin{equation}
\label{eq: 17}
     \RR'(\Xoo) = \frac{1}{e^{\epsor} + 1}U(\Xoz) + \frac{e^{\epsor}}{e^{\epsor} + 1}U(\Xoo).
\end{equation}

\textcolor{DarkBlue}{Here, we diverge from the proof in~\cite{lowy2021private}. We denote $\tepso := n\epsor$ and \textcolor{blue}{$\tdelo := \delor.$} Then, by the assumption of $\RR(X)$, for any $X$, we have $\RR(X) \tedosim \RR(\Xoz))$ and $\RR(X) \tedosim \RR(\Xoo))$. This is because by the assumption, when the dataset changes from any $X$ to $\Xoz$ (or $\Xoo$), the total change in terms of $l_2$ norm can be $n$ times that under an adjacent pair. Thus, one has to scale the $\epsilon_0^r$ by $n$ while keeping the same $\delta_0^r$. }

Now, we resume the same idea as in~\cite{lowy2021private}.
By convexity of hockey-stick divergence and the above result, we have $\RR(X) \tedosim \frac{1}{2}(\RR(\Xoz) + \RR(\Xoo)) := \rho$ for all $X \in \XX^n.$ 
That is, $\RR$ is $(\tepso, \tdelo)$ deletion group DP for groups of size $n$ with reference distribution $\rho.$ Thus, by~\cite[Lemma D.11]{lowy2021private}, we have that there exists a local randomizer $\RR''$ such that $\RR''(X)$ and $\rho$ are $(\tepso, 0)$ indistinguishable and $TV(\RR''(X), \RR(X)) \leq \tdelo$ for all $X.$ Then by the definition of $(\tepso, 0)$ indistinguishability, for all $X$ there exists a ``left-over'' distribution $LO(X)$ such that $\RR''(X) = \frac{1}{e^{\tepso}} \rho + (1 - 1/e^{\tepso})LO(X) = \frac{1}{2e^{\tepso}}(\RR(\Xoz) + \RR(\Xoo)) + (1 - 1/e^{\tepso})LO(X).$ 

Now, define a randomizer $\LL$ by $\LL(\Xoz) := \RR(\Xoz), ~\LL(\Xoo) := \RR'(\Xoo),$ and \begin{align}
\label{eq: 18}
\LL(X) &:= \frac{1}{2e^{\tepso}}\RR(\Xoz) + \frac{1}{2e^{\tepso}}\RR'(\Xoo) + (1 - 1/e^{\tepso})LO(X) \nonumber \\
&= \frac{1}{2e^{\tepso}}U(\Xoz) + \frac{1}{2e^{\tepso}}U(\Xoo) + (1 - 1/e^{\tepso})LO(X)
\end{align}
for all $X \in \XX^n \setminus \{\Xoz, \Xoo\}.$ (The equality follows from \eqref{eq: 16} and \eqref{eq: 17}.)
Note that $TV(\RR(\Xoz), \LL(\Xoz)) = 0, ~TV(\RR(\Xoo), \LL(\Xoo)) \leq \delor,$ and for all $X \in \XX^n \setminus \{\Xoz, \Xoo\}$, \textcolor{DarkBlue}{ $TV(\RR(X), \LL(X)) \leq TV(\RR(X), \RR''(X)) + TV(\RR''(X), \LL(X)) \leq \tdelo + \frac{1}{2e^{\tepso}}TV(\RR'(\Xoo), \RR(\Xoo)) = (1 + \frac{1}{2e^{n\epsor}})\delor$}.

Keeping $r$ fixed (omitting $r$ scripts everywhere), for any $i \in [N]$ and $\bz := \bz_{1:r-1} \in \ZZ^{(r-1) \times N},$ let $\LL^{(i)}(\bz, \cdot)$, ~$U^{(i)}(\bz, \cdot)$, and $LO^{(i)}(\bz, \cdot)$ denote the randomizers resulting from the process described above. Let $\Al_{\LL}: \XX^{n \times N} \to \ZZ^{N}$ be defined exactly the same way as $\Al_s^r := \Al_s$ (same $\pi$) but with the randomizers $\rand$ replaced by $\LL^{(i)}$. Since $\Al_s$ applies each randomizer $\RR^{(i)}$ exactly once and $\RR^{(1)}(\bz, X_{\pi(1)}, \cdots \RR^{(N)}(\bz, X_{\pi(N)})$ are independent (conditional on $\bz = \bz_{1:r-1}$) \footnote{This follows from the assumption  that $\RR^{(i)}(\bz_{1:r-1}, X)$ is conditionally independent of $X'$ given $\bz_{1:r-1}$ for all $\bz_{1:r-1}$ and $X \neq X'$.}, we have 
\textcolor{DarkBlue}{$TV(\Al_s(\bx_0), \Al_{\LL}(\bx_0)) \leq N (1 + \frac{1}{2e^{n\epsor}})\delor$ and $TV(\Al_s(\bx_1), \Al_{\LL}(\bx_1) \leq N (1 + \frac{1}{2e^{n\epsor}})\delor$.} Now we claim that $\Al_{\LL}(\bx_0)$ and $\Al_{\LL}(\bx_1)$ are $(\epsr, \delta)$ indistinguishable for any $\delta \geq 2e^{-Ne^{-n\epsor}/16}.$ Observe that this claim implies that $\Al_s(\bx_0)$ and $\Al_s(\bx_1)$ are $(\epsr, \delta^r)$ indistinguishable by~\cite[Lemma D.13]{lowy2021private} (with $P':= \Al_{\LL}(\bx_0), Q':= \Al_{\LL}(\bx_1), P:= \Al_{s}(\bx_0), Q:= \Al_{s}(\bx_1)$.) Therefore, it only remains to prove the claim, i.e. to show that $D_{e^{\epsr}}(\mathcal{A}_{\mathcal{L}}(\bx_0),\mathcal{A}_{\mathcal{L}}(\bx_1) \leq \delta$ for any $\delta \geq  2e^{-Ne^{-n\epsor}/16}.$

Now, 
define $\LL_U^{(i)}(\bz, X):= \begin{cases}
U^{(i)}(\bz, \Xoz) &\mbox{if} ~X = \Xoz \\
U^{(i)}(\bz, \Xoo) &\mbox{if} ~X = \Xoo \\
\LL^{(i)}(\bz, X) &\mbox{otherwise}.
\end{cases}.$
For any inputs $\bz, \bx$, let $\Al_{U}(\bz, \bx)$ be defined exactly the same as $\Al_s(\bz, \bx)$ (same $\pi$) but with the randomizers $\rand$ replaced by $\LL_U^{(i)}$. Then by \eqref{eq: 16} and \eqref{eq: 17}, 
\begin{equation}
    \label{eq: 19}
    \Al_{\LL}(\bx_0) = \frac{e^{\epsor}}{e^{\epsor} + 1}\Al_{U}(\bx_0) + \frac{1}{e^{\epsor} + 1}\Al_U(\bx_1) ~\text{and} ~\Al_{\LL}(\bx_1) = \frac{1}{e^{\epsor} + 1}\Al_{U}(\bx_0) + \frac{e^{\epsor}}{e^{\epsor} + 1}\Al_U(\bx_1).
\end{equation}

Then by \eqref{eq: 18}, for any $X \in \XX^n \setminus \{\Xoz, \Xoo\}$ and any $\bz = \bz_{1:r-1} \in \ZZ^{(r-1) \times N},$ we have $\LL_U^{(i)}(\bz, X) = \frac{1}{2e^{\tepso}}\LL^{(i)}_U(\bz, \Xoz) + \frac{1}{2e^{\tepso}}\LL^{(i)}_U(\bz, \Xoo) + (1 - e^{-\tepso})LO^{(i)}(\bz, X).$ Hence,~\cite[Lemma D.10]{lowy2021private} (with $p:= e^{-\tepso} = e^{-n\epsor}$) implies that $\Al_U(\bx_0)$ and $\Al_U(\bx_1)$) are \[
\left(\log\left(1 + \frac{8 \sqrt{e^{\widetilde{\epso}} \ln(4/\delta)}}{\sqrt{N}} + \frac{8 e^{\widetilde{\epso}}}{N}\right), \delta\right)
\] 
indistinguishable for any $\delta \geq 2e^{-Ne^{-n\epsor}/16}.$ 

\textcolor{DarkBlue}{Here, we also slightly diverge from~\cite{lowy2021private}. Instead of using~\cite[Lemma D.14]{lowy2021private}\footnote{We think that its restatement of~\cite[Lemma 2.3]{feldman2022hiding} is not correct (which can be easily fixed though).}, we can directly follow the proof of Lemma 3.5 in~\cite{feldman2022hiding} and Lemma 2.3 in~\cite{feldman2022hiding} to establish our claim that  $\Al_{\LL}(\bx_0)$ and $\Al_{\LL}(\bx_1)$ are indistinguishable (hence the final result). Here, we also slightly improve the $\delta$ term compared to~\cite{feldman2022hiding} by applying amplification via sub-sampling to the $\delta$ term as well. In particular, the key step is to rewrite~\eqref{eq: 19} as follows (with $T:= \frac{1}{2} ({\Al_{U}}(\bx_0) + {\Al_{U}}(\bx_1))  $
\begin{equation}
    \label{eq:with-T}
    \Al_{\LL}(\bx_0) = \frac{2}{e^{\epsor} + 1}T + \frac{e^{\epsor} - 1}{e^{\epsor} + 1}\Al_U(\bx_0) ~\text{and} ~\Al_{\LL}(\bx_1) = \frac{2}{e^{\epsor} + 1}T + \frac{e^{\epsor}-1}{e^{\epsor} + 1}\Al_U(\bx_1).
\end{equation}
Thus, by the convexity of the hockey-stick divergence and Lemma 2.3 in~\cite{feldman2022hiding}, we have $\Al_{\LL}(\bx_0)$ and $\Al_{\LL}(\bx_1)$ are \[
\left(\log\left(1 + \frac{\epsor-1}{\epsor+1}\left(\frac{8 \sqrt{e^{\widetilde{\epso}} \ln(4/\delr)}}{\sqrt{N}}\right) + \frac{8 e^{\widetilde{\epso}}}{N}\right), \frac{\epsor-1}{\epsor+1}\delta\right)
\] 
indistinguishable for any $\delta \geq 2e^{-Ne^{-n\epsor}/16}.$ As decribed before, this leads to the result that $\Al_s(\bx_0)$ and $\Al_s(\bx_1)$ are $(\epsr, \delta^r)$ indistinguishable by~\cite[Lemma D.13]{lowy2021private} (original result in Lemma 3.17 of~\cite{dwork2014algorithmic}) with (noting that $\widetilde{\epso} = n \epsilon_0^r$)
\begin{align*}
   \epsilon^r &:= \ln\left[1 + \left(\frac{e^{\epsor} - 1}{e^{\epsor} + 1}\right)\left(\frac{8 \sqrt{e^{n\epsor}\ln(4/\delta)}}{\sqrt{N}} + \frac{8 e^{n\epsor}}{N} \right) 
    \right],\\
   \delta^r &:= \left(\frac{e^{\epsor} - 1}{e^{\epsor} + 1}\right)\delta + N (e^{\epsilon^r}+1)(1+e^{-\epsilon_0^r}/2)\delor.
\end{align*}
}
\end{proof}

\section{Further Discussions}
\label{app:concluding}
In this section, we provide more details on our privacy notion and algorithm design. 

\subsection{Silo-level LDP/SDP vs. Other Privacy Notions}
\label{app:ldp}
In this section, we compare our silo-level LDP and SDP with standard privacy notions for single-agent LCBs, including local, central, and shuffle model for DP, respectively. 

\textbf{Silo-level LDP vs. single-agent local DP.} Under standard LDP for single-agent LCBs~\cite{zheng2020locally,duchi2013local,Zhou_Tan_2021}, each user only trusts herself and hence privatizes her response before sending it to the agent. In contrast, under silo-level LDP, each local user trusts the local silo (agent), which aligns with the pratical situations of cross-silo FL, e.g., patients often trust the local hospitals. In such cases, standard LDP becomes unnecessarily stringent, hindering performance/regret and making it less appealing to cross-silo federated LCBs.

\textbf{Silo-level LDP vs. single-agent central DP.} The comparison with standard central DP for single-agent LCB (e.g.,~\cite{shariff2018differentially}) is delicate. We first note that under both notions, users trust the agent and the privacy burden lies at the agent. Under standard central DP, the agent uses private statistics until round $t$ to choose action for each round $t$, which ensures that any other users $t'\neq t$ cannot infer too much about user $t$'s information by observing the actions on rounds $t'\neq t$ (i.e., joint differential privacy (JDP)~\cite{kearns2014mechanism})\footnote{As shown in~\cite{shariff2018differentially}, JDP relaxation is necessary for achieving sub-linear regret for LCBs under the central model.}. On the other hand, silo-level LDP does not necessarily require each agent (silo) to use private statistics to recommend actions to users within the silo. Instead, it only requires the agent to privatize its sent messages (both schedule and content). Thus, silo-level LDP may not protect a user $t$ from the colluding of all other users within the \emph{same} silo. In other words, the adversary model for silo-level LDP is that the adversary could be any other silos or the central server rather than other users within the same silo. Note that the same adversary model is assumed in a similar notion for federated supervised learning (e.g., inter-silo record-level differential privacy (ISRL-DP) in~\cite{lowy2021private}). In fact, with a minor tweak of our Algorithm~\ref{alg:FedLUCB-SDP}, one can achieve a slightly stronger notion of privacy than silo-level LDP in that it now can protect against both other silos/server and users within the same silo. The key idea is exactly that now each agent will only use private statistics to recommend actions, see Appendix~\ref{app:tweak}.

\textbf{Silo-level LDP vs. Federated DP in~\cite{dubey2020differentially}.} In~\cite{dubey2020differentially}, the authors define the so-called notion of \emph{federated DP} for federated LCBs, which essentially means that ``the action
chosen by any agent must be sufficiently impervious (in probability) to any single  data from any other agent''. This privacy guarantee is directly implied by our silo-level LDP. In fact, in order to show such a privacy guarantee, ~\cite{dubey2020differentially} basically tried to show that the outgoing communication is private, which is the idea of silo-level LDP. However, as mentioned in the main paper,~\cite{dubey2020differentially} only privatizes the communicated data and fails to privatize the communication schedule, which leads to privacy leakage. Moreover, as already mentioned in Remark~\ref{rem:weak-FedDP}, Fed-DP fails to protect a user's privacy even under a reasonable adversary model. Thus, we believe that silo-level LDP is a better option for federated LCBs.

\textbf{SDP vs. single-agent shuffle DP.} Under the single-agent shuffle DP~\cite{pmlr-v162-chowdhury22a,tenenbaum2023concurrent}, the shuffler takes as input a batch of users' data (i.e., from $t_1$ to $t_2$), which enables to achieve a regret of $\widetilde{O}(T^{3/5})$ (vs. $\widetilde{O}(T^{3/4})$ regret under local model and $\widetilde{O}(\sqrt{T})$ regret under central model). In contrast, under our SDP, the shuffler takes as input the DP outputs from all $M$ agents. Roughly speaking, single-agent shuffle DP aims to amplify the privacy dependence on $T$ while our SDP amplifies privacy over $M$. Due to this, single-agent shuffle DP can directly apply a standard amplification lemma (e.g.,~\cite{feldman2022hiding}) or shuffle protocol (e.g.,~\cite{cheu2021shuffle}) that works well with LDP mechanism at each user (i.e., the size of dataset is $n=1$). In contrast, in order to realize amplification over $M$ agents' DP outputs, we have to carefully modify the standard amplification lemma to handle the fact that now each local mechanism operates on $n>1$ data points, which is one of the key motivations for our new amplification lemma.

\subsection{A Simple Tweak of Algorithm~\ref{alg:FedLUCB-SDP} for a Stronger Privacy Guarantee}
\label{app:tweak}
As discussed in the last subsection, the adversary model behind silo-level LDP only includes other silos and the central server, i.e., excluding adversary users within the same silo. Thus, for silo-level LDP, Algorithm~\ref{alg:FedLUCB-SDP} can use non-private data to recommend actions within a batch (e.g., $V_{t,i}$ includes non-private recent local bias vectors and covariance matrices). If one is also interested in protecting against adversary users within the same silo, a simple tweak of Algorithm~\ref{alg:FedLUCB-SDP} suffices.

\begin{algorithm}[!t]
  \caption{Priv-FedLinUCB-Lazy }
  \label{alg:FedLUCB-lazy}
\begin{algorithmic}[1]
  \STATE {\bfseries Parameters:} Batch size $B \in \mathbb{N}$, regularization $\lambda >0$, confidence radii $\{\beta_{t,i}\}_{t\in [T], i\in [M]}$, feature map $\phi_i:\cC_i \times \cK_i \to \Real^d$, privacy protocol $\cP = (\cR,\cS,\cA)$
  \STATE {\bfseries Initialize:} For all $i \in [M]$, $W_{i} = 0, U_{i} = 0$, $\widetilde{W}_{\text{syn}} = 0$, $\widetilde{U}_{\text{syn}} = 0$ 
 \FOR{$t\!=\!1, \ldots, T$}
  
    \FOR{each agent $i = 1,\ldots, M$}
      \STATE \textcolor{DarkGreen}{$V_{t,i} =  \lambda I+ \widetilde{W}_{\text{syn}}$, $\hat{\theta}_{t,i} = V_{t,i}^{-1}\widetilde{U}_{\text{syn}}$}
      \STATE Play arm $a_{t,i} \!=\! \argmax_{a \in \cK_i} \inner{\phi_i(c_{t,i}, a)} {\hat{\theta}_{t,i}} + \beta_{t,i} \norm{\phi_i(c_{t,i},a)}_{V_{t,i}^{-1}}$ and set $x_{t,i} = \phi_i(c_{t,i}, a_{t,i})$
      \STATE Observe reward $y_{t,i}$
      \STATE Update  $U_{i} = U_{i} + x_{t,i}y_{t,i}$ and $W_{i} = W_{i} + x_{t,i}x_{t,i}^{\top}$
      \ENDFOR
      \IF{$t \Mod B = 0$}
       \STATE{\textcolor{DarkBlue}{\texttt{// Local randomizer $\cR$ at \emph{all} agents $i \in [M]$ }}}
        \STATE Send randomized messages $R_{t,i}^{\text{bias}} = \cR^{\text{bias}}(U_i)$ and $R_{t,j}^{\text{cov}} = \cR^{\text{cov}}(W_i)$ to the shuffler
         \STATE{\textcolor{DarkBlue}{\texttt{// Third party $\cS$ }}}
         \STATE $S_t^{\text{bias}} = \cS(\{R_{t,i}^{\text{bias}}\}_{i \in [M]})$ and $S_t^{\text{cov}} = \cS(\{R_{t,i}^{\text{cov}}\}_{i \in [M]})$
          \STATE{\textcolor{DarkBlue}{\texttt{// Analyzer $\cA$ at the server}}}
         \STATE Construct private cumulative statistics $\widetilde{U}_{\text{syn}}= \cA^{\text{bias}}(S_t^{\text{bias}})$ and $\widetilde{W}_{\text{syn}}= \cA^{\text{cov}}(S_t^{\text{cov}})$
         \STATE{\textcolor{DarkBlue}{\texttt{// \emph{All} agents $i \in [M]$}}}
        \STATE Receive $\widetilde{W}_{\text{syn}}$ and $\widetilde{U}_{\text{syn}}$ from the server
        \STATE Reset $W_{i} = 0$, $U_{i} = 0$
      \ENDIF

  \ENDFOR
\end{algorithmic}
\end{algorithm}

As shown in Algorithm~\ref{alg:FedLUCB-lazy}, the only difference is a lazy update of $\hat{\theta}_{t,i}$ is adopted (line 5), i.e., it is only computed using private data without any dependence on new non-private local data. In fact, same regret bound as in Theorem~\ref{thm:LDP} can be achieved for this new algorithm (though empirical performance could be worse due to the lazy update). 
In the following, we highlight the key changes in the regret analysis. It basically follows the six steps in the proof of Lemma~\ref{lem:general}. One can now define a mapping $\kappa (t)$ that maps any $t \in [T]$ to the most recent communication round. That is, for any $t \in [t_{k-1}, t_{k}]$ where $t_k = kB$ is the $k$-th communication round, we have $\kappa(t) = t_{k-1}$. Then, one can replace all $t$ in $V_{t,i}$ and $G_{t,i}$ by $\kappa(t)$. The main difference that needs a check is Step 4 when bounding the regret in good epochs. The key is again to establish a similar form as~\eqref{eq:det-trick}. To this end, note that for all $t \in [t_{k-1}, t_k]$ $V_{k} \succeq \bar{V}_{t,i}$ and $G_{\kappa(t),i} + \lambda_{\min} I = V_{k-1}$, which enables us to obtain $ \norm{x_{t,i}}_{(G_{\kappa(t),i} + \lambda_{\min} I)^{-1}} \le \sqrt{2}\norm{x_{t,i}}_{\bar{V}_{t,i}^{-1}}$. Following the same analysis yields the desired regret bound.

\subsection{Non-unique Users}
\label{app:non-unique}
In the main paper, we assume all users across all silos and $T$ rounds are unique. Here, we briefly discuss how to handle the case of non-unique users.
\begin{itemize}
    \item The same user appears multiple times in the same silo. One example of this could be one patient visiting the same hospital multiple times. In such cases, one needs to carefully apply group privacy or other technique (e.g.,~\cite{pmlr-v162-chowdhury22a}) to characterize the privacy loss of these returning users.
    \item The same user appears multiple times across different silos. One example of this could be one patient who has multiple records across different hospitals. Then, one needs to use adaptive advanced composition to characterize the privacy loss of these returning users.
\end{itemize}

\section{Additional Details on Simulation Results}\label{app:sim}

\begin{figure}[t]
		\begin{subfigure}[t]{.45\linewidth}
		\centering
        \includegraphics[width = 2.5in]{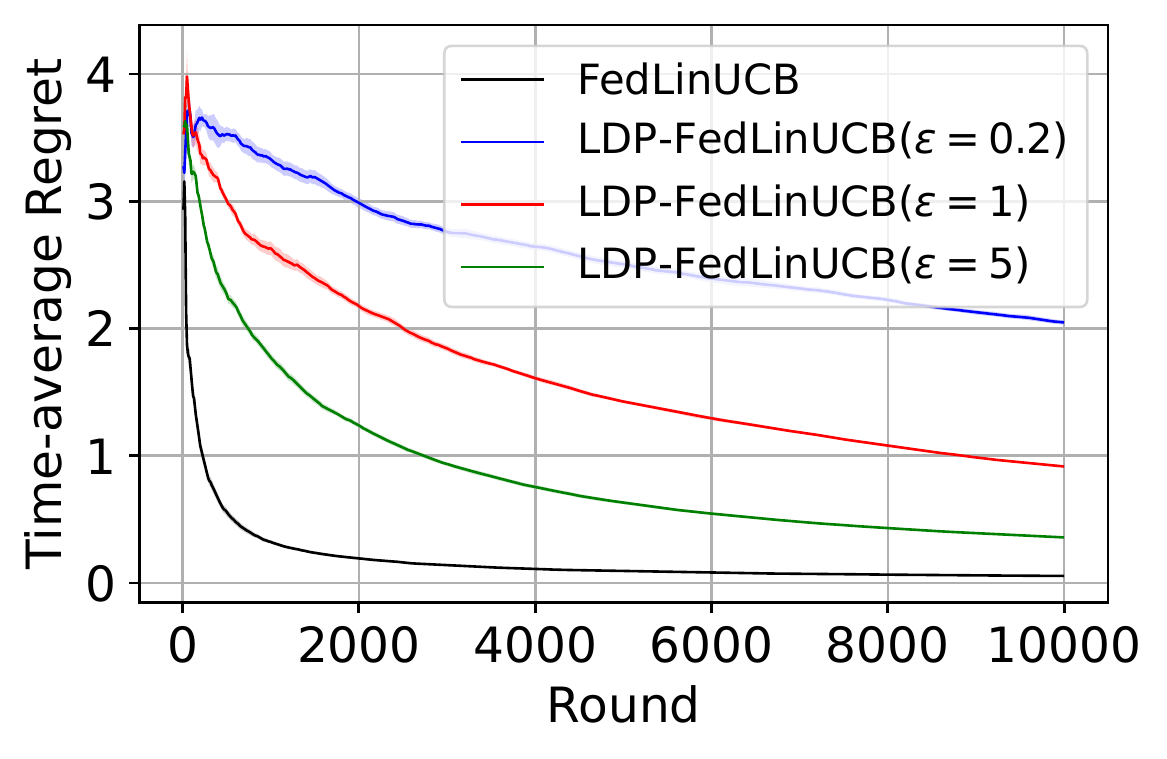}
        \vskip -1mm
			\caption{Synthetic data (varying $\epsilon$, $\delta=0.1)$}
		\end{subfigure} \ \
  \begin{subfigure}[t]{.45\linewidth}
        \centering
		\includegraphics[width = 2.5in]{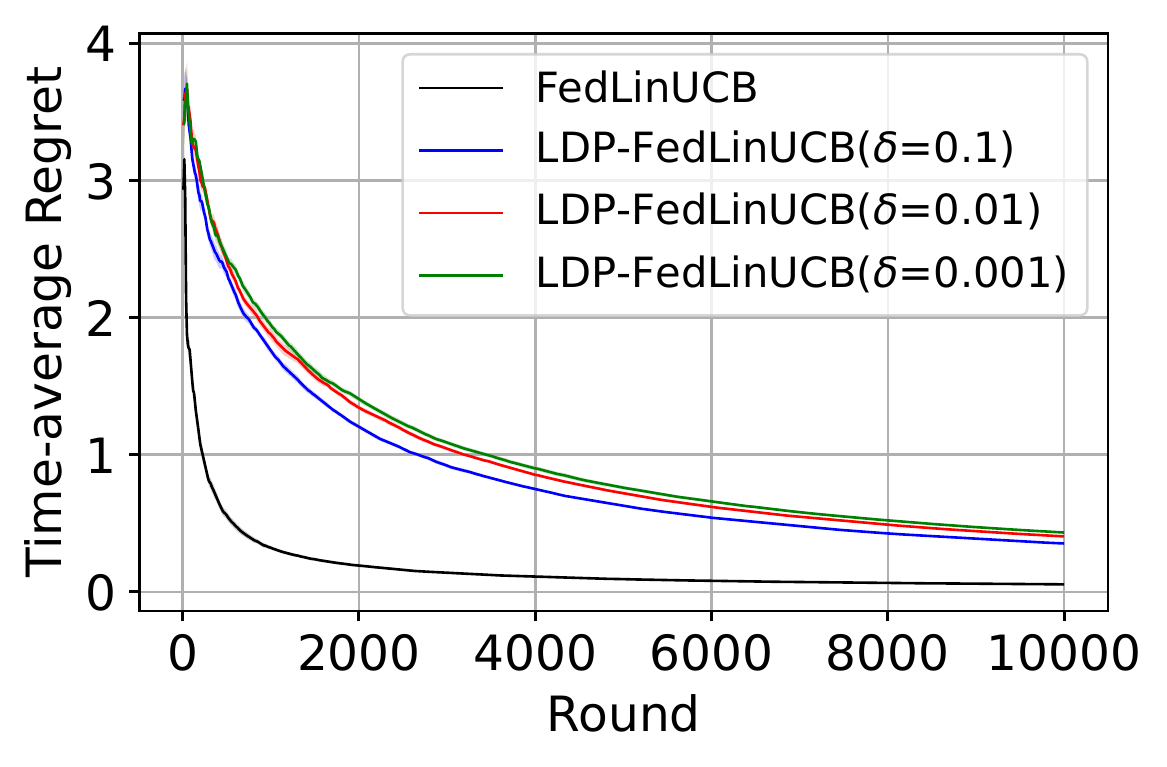}
  \vskip -1mm
			\caption{Synthetic data (varying $\delta$, $\epsilon=5$) }
		\end{subfigure} \\
  \begin{subfigure}[t]{.45\linewidth}
        \centering
		\includegraphics[width = 2.5in]{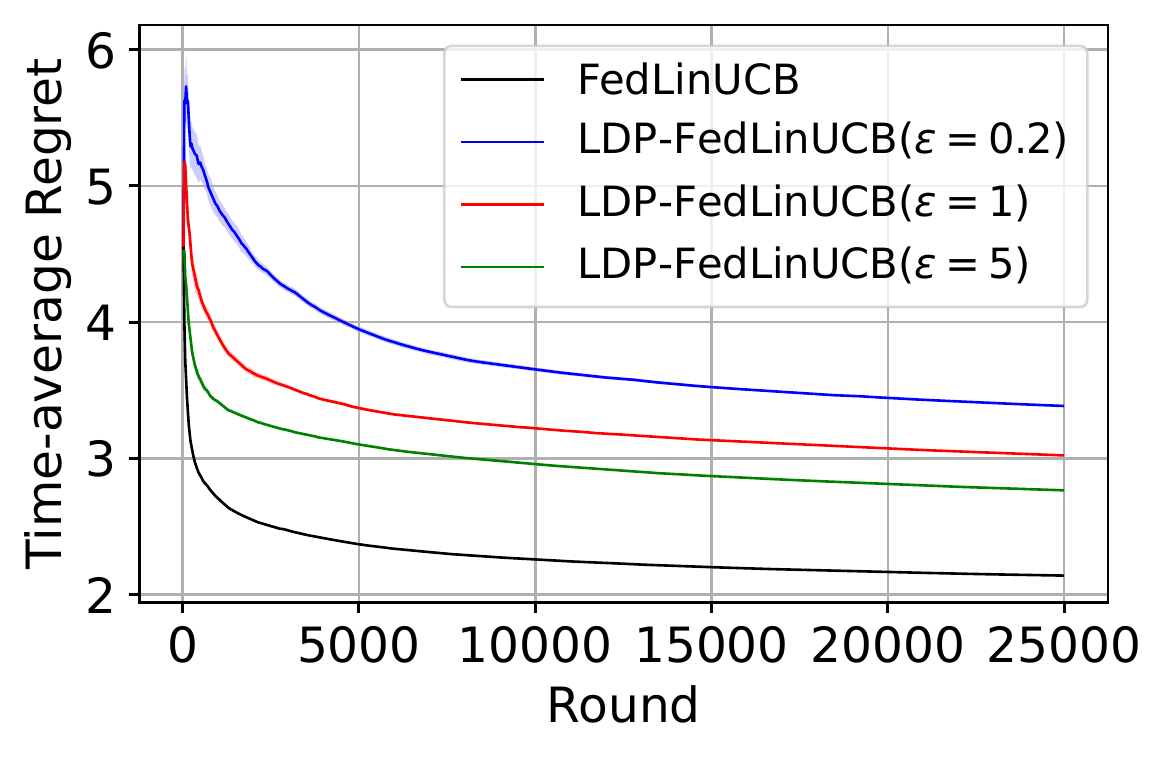}
  \vskip -1mm
			\caption{Real data (varying $\epsilon$, $\delta=0.1)$ }
		\end{subfigure}\ \
  \begin{subfigure}[t]{.42\linewidth}
		\centering
	\includegraphics[width = 2.5in]{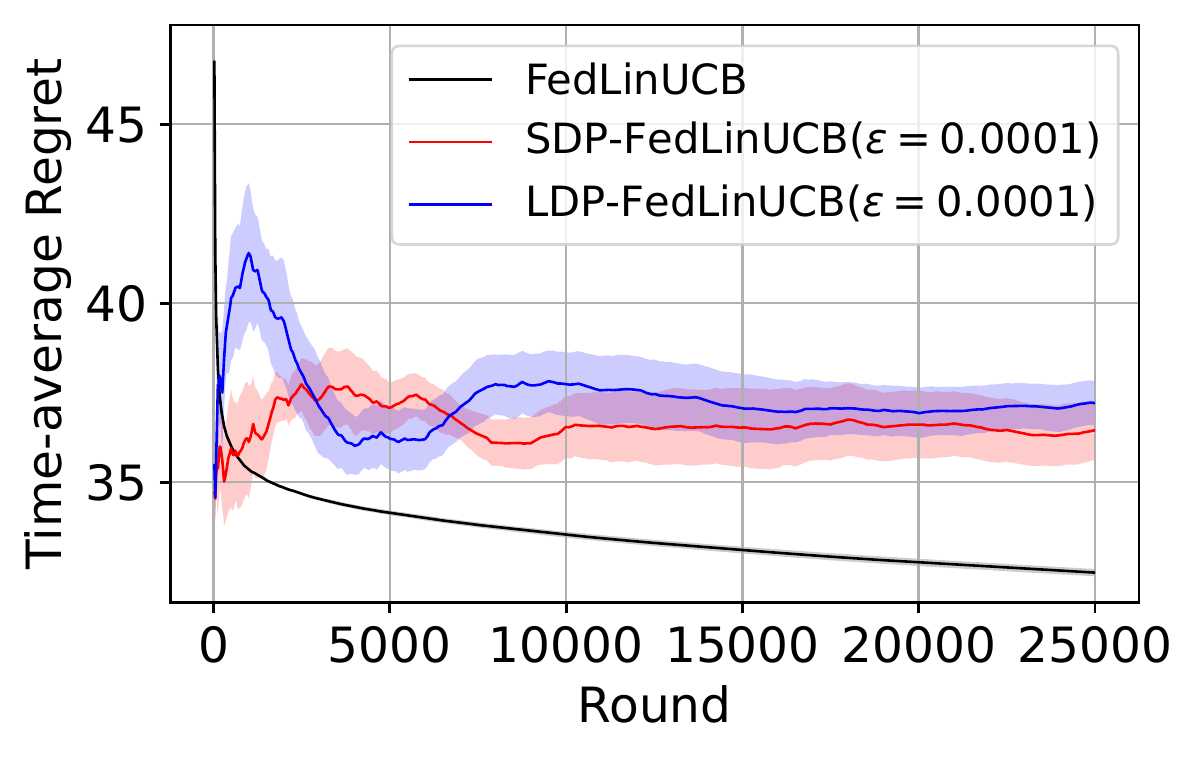}
 \vskip -1mm
			\caption{Real data ($M=100$)}
		\end{subfigure}
   \vspace{-1mm}
\caption{\footnotesize{Comparison of time-average group regret for FedLinUCB (non-private) and LDP-FedLinUCB (i.e., under silo-level LDP) on (a, b) synthetic Gaussian bandit instance and (c,d) bandit instance generated from MSLR-WEB10K Learning to Rank dataset.}  }\label{fig:all_algos_app}
		\vspace{0mm}
\end{figure}

In Figure~\ref{fig:all_algos_app}, we compare regret performance of LDP-FedLinUCB with FedLinUCB under varying privacy budgets.\footnote{All existing non-private federated LCB algorithms (e.g., \citet{wang2019distributed}) adopts adaptive communication. We refrain from comparing with those to maintain consistency in presentation.} 
In sub-figure (a), we plot results for $\delta=0.1$ and varying level of $\epsilon \in \lbrace 0.2,1,5\rbrace$ on synthetic Gaussian bandit instance, wherein sub-figure (b), we plot results for $\epsilon=5$ and varying level of $\delta \in \lbrace 0.1,0.01,0.001\rbrace$. In sub-figure (c), we plot results for $\delta=0.1$ and varying level of $\epsilon \in \lbrace 0.2,1,5\rbrace$ on bandit instance generated from MSLR-WEB10K data by training a lasso model on bodyfeatures ($d=78$). In all these plots, we observe that regret of LDP-FedLinUCB decreases and, comes closer to that of FedLinUCB as $\epsilon,\delta$ increases (i.e., level of privacy protection decreases), which support our theoretical results. Here, we don't compare SDP-LinUCB (with privacy amplification) since its privacy guarantee holds for $\epsilon,\delta \ll 1$. Instead, we do so in sub-figure (d) with $\epsilon=\delta=0.0001$. Here also, we observe a drop in regret
of SDP-FedLinUCB compared to that of LDP-FedLinUCB.

\end{document}